\documentclass[10pt,journal,compsoc]{IEEEtran}

\usepackage[sort,numbers]{natbib}

\ifCLASSINFOpdf
\else
\fi

\usepackage{amsmath,amsfonts,bm}

\def\eqref#1{equation~\ref{#1}}

\def\1{\bm{1}}

\DeclareMathAlphabet{\mathsfit}{\encodingdefault}{\sfdefault}{m}{sl}
\SetMathAlphabet{\mathsfit}{bold}{\encodingdefault}{\sfdefault}{bx}{n}

\usepackage{hyperref}
\usepackage{url}
\usepackage{graphicx}
\usepackage{ulem}
\usepackage{multirow}
\usepackage[show]{notes-alt}

\usepackage{booktabs}
\usepackage{colortbl} %
\usepackage{nicematrix} %
\usepackage{float}
\usepackage{pifont}
\usepackage[noend]{algpseudocode}
\usepackage{algorithmicx,algorithm}
\newcommand{\mpara}[1]{\medskip\noindent{\bf #1}}
\usepackage{bbm}
\usepackage{caption}

\usepackage{subfig}

\usepackage{amsthm}
\newtheorem{RQ}{RQ}
\newtheorem{theorem}{Theorem}
\newtheorem{proposition}{Proposition}
\newtheorem{definition}{Definition}

\newcommand{\approach}{\textsc{Zorro}}

\newcommand{\gnnexp}{GNNExplainer}
\newcommand{\fidelity}{RDT-Fidelity}

\usepackage{rotating}
\usepackage{makecell}

\graphicspath{images/}

\usepackage{hyperref}

\begin{document}

\title{\approach{}: Valid, Sparse, and Stable Explanations in Graph Neural Networks}

\author{Thorben~Funke,
        Megha~Khosla,
        Mandeep Rathee,
        and~Avishek~Anand%
\IEEEcompsocitemizethanks{\IEEEcompsocthanksitem T. Funke and M. Rathee  are with the L3S Research Center, Leibniz University Hanover,  Germany.
\protect\\
Email: tfunke@l3s.de
\IEEEcompsocthanksitem M. Khosla and A. Anand are with TU Delft, Netherlands.
}%
\thanks{Manuscript received April 19, 2005; revised August 26, 2015.}
}

\markboth{Journal of \LaTeX\ Class Files,~Vol.~14, No.~8, August~2015}%
{Shell \MakeLowercase{\textit{et al.}}: Bare Advanced Demo of IEEEtran.cls for IEEE Computer Society Journals}

\IEEEtitleabstractindextext{%
\begin{abstract}%
With the ever-increasing popularity and applications of graph neural networks, several proposals have been made to explain and understand the decisions of a graph neural network.
Explanations for graph neural networks differ in principle from other input settings. 
It is important to attribute the decision to input features and other related instances connected by the graph structure.  
We find that the previous explanation generation approaches that maximize the mutual information between the label distribution produced by the model and the explanation to be restrictive. 
Specifically, existing approaches do not enforce explanations to be valid, sparse, or robust to input perturbations. 
In this paper, we lay down some of the fundamental principles that an explanation method for graph neural networks should follow and introduce a metric \textit{\fidelity{}} as a measure of the explanation's effectiveness.
We propose a novel approach Zorro based on the principles from \textit{rate-distortion theory} that uses a simple combinatorial procedure to optimize for \fidelity{}.
Extensive experiments on real and synthetic datasets reveal that Zorro produces sparser, stable, and more faithful explanations than existing graph neural network explanation approaches.  \end{abstract}

\begin{IEEEkeywords}
Explainability, Graph Neural Networks, Interpretability.
\end{IEEEkeywords}}

\maketitle

\IEEEdisplaynontitleabstractindextext

\IEEEpeerreviewmaketitle

\ifCLASSOPTIONcompsoc
\IEEEraisesectionheading{\section{Introduction}\label{sec:introduction}}
\else
\section{Introduction}
\label{sec:introduction}
\fi

Graph neural networks (GNNs) are a flexible and powerful family of models that build representations of nodes or edges on irregular graph-structured data and have experienced significant attention in recent years. 
GNNs are based on the ``neighborhood aggregation'' scheme, where a node representation is learned by aggregating features from their neighbors. 
Learning complex neighborhood aggregations and latent feature extraction has enabled GNNs to achieve state-of-the-art performance on node and graph classification tasks.
This complexity, on the other hand, leads to a more opaque and non-interpretable model.
To alleviate the problem of interpretability, we focus on explaining the rationale underlying a given prediction of already trained GNNs. 

\begin{figure*}
    \centering
    \includegraphics[width=0.87\textwidth]{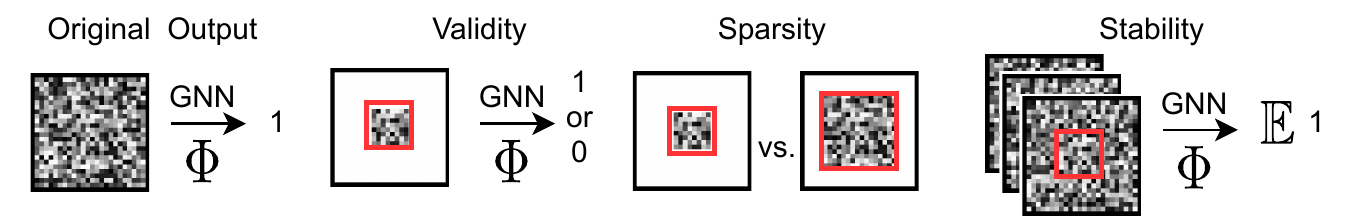}
    \caption{Illustration of validity, sparsity, and stability. The GNN $\Phi$ takes the feature matrix $X$, which is illustrated as a grayscale matrix, and the relations from the graph $G$, which is not shown for simplicity, to predict the class label (1). An explanation selects the most important inputs from the feature matrix responsible for the prediction, which we illustrate as red rectangles. The validity of an explanation is the property to preserve the prediction if a fixed baseline value replaces all not selected values. The sparsity is the number of selected elements, where fewer elements are desirable. Lastly, stability is the property to preserve the prediction if all not selected values are perturbed.
    Existing methods only optimize for validity and sparsity.
    However, even trivial explanations can be valid and sparse. 
    We introduce \textit{\fidelity{}} as the new measure, which also accounts for explanation stability, and present \approach{} to retrieve valid, sparse, and stable explanations. 
    }
    \label{fig:overview}
\end{figure*}

There are diverse notions and regimes of explainability and interpretability for machine learning models -- (a) interpretable models vs.\ post-hoc explanations, (b) model-introspective vs.\ model-agnostic explanations, (c) outputs in terms of feature vs.\ data attributions~\cite{ribeiro2016:lime,lundberg2017:shap}. In this work, we aim to explain the decision of an already trained GNN model, i.e., \textit{compute post-hoc explanations for a trained GNN}.
Additionally, we do not assume any access to the trained model parameters, i.e., we are model-agnostic or black-box regime.
Finally, our explanation attributes the reason for an underlying GNN prediction to either a subset of features or neighboring nodes or both.

There has been recent interest in designing explainers for GNNs that produce feature attributions in a post-hoc manner~\citep{ying2019:gnnexplainer, vu2020pgm, sanchez2020evaluating} where a combination of nodes, edges, or features is retrieved as an explanation.
We introduce some essential notions of validity, sparsity, and stability for explaining GNNs and argue that many of the existing works on explainable GNNs do not satisfy these principles. 
To systematically fill in the above gaps, we commence by formulating three desired properties of a GNN explanation: \textit{validity}, \textit{sparsity}, and \textit{stability}. Figure~\ref{fig:overview} provides an illustration of these three properties.

\mpara{Validity.} Existing explanations approaches used for explaining GNNs like gradient-based feature attribution techniques select nodes or features~\citep{sundararajan2017:integratedgrad} are not optimized to be valid as well as being explanatory. 
An explanation is valid if just the explanation (a subset of features and nodes) as input would be sufficient to arrive at the same prediction.

\mpara{Sparsity.} It is easy to see that validity alone is not sufficient for an explanation as the entire input is a valid explanation~\citep{wolf2019formal}. Ideally, the explanation should only highlight those parts of the input with the highest discriminative information. 
Existing explanation approaches accomplish this by outputting distributions or \textit{soft-masks} over input features or nodes~\citep{ying2019:gnnexplainer}.
However, humans find it hard to make sense of soft masks and instead prefer sparse binary masks or \textit{hard masks}~\citep{baan2019transformer, pruthi2019learning,kindermans2019reliability,zhang2019dissonance}. 
We define \textit{sparsity} as the \textit{size of the explanation in terms of number of non-zero elements} in the explanation. 
A sparse explanation in the form of a hard mask is, therefore, more desirable and reduces ambiguities due to soft masks~\citep{jacovi2020aligning}.

\mpara{Stability.} Validity and sparsity, though necessary, are not sufficient to define an explanation. 
In Section~\ref{sec:valsparse}, we show that the trivial empty explanation (all features are replaced by $0$s) could be a valid explanation for many cases. 
The high validity observed in such cases is an artifact of a particular configuration of trained model parameters. 
We would rather expect that the model retains its predicted class with only the knowledge of the explanation \textit{while the rest of the information in the input is filled randomly}. In other words, an explanation should be valid independent of the rest of the input. 
We say that an explanation is \textit{stable} if the behavior of the GNN is unaffected by the features outside of the explanation.
Most of the existing works do not consider stability in their modeling of explanation approaches.

In this paper, we introduce a new metric called \textit{\fidelity{}} which is grounded in the principles of \textit{rate-distortion theory} and reflects these three desiderata into a single measure. 
Essentially, we cast the problem of finding explanations given a trained model as a signal/message reconstruction task involving a sender, a receiver, and a noisy channel. 
The message sent by the sender is the actual feature vector, with the explanation being a subset of immutable feature values. 
The noisy channel can obfuscate only the features that do not belong to the explanation. The explanation's \fidelity{} lies in the degree to which the decoder can faithfully reconstruct from the noisy feature vector. 
Optimizing \fidelity{} is NP-hard, and we consequently propose a greedy combinatorial procedure \approach{} that generates sparse, valid, and stable explanations. 

Accurately measuring the actual effectiveness of post-hoc explanations has been acknowledged to be a challenging problem due to the lack of explanation ground truth.
We carry out an extensive and comprehensive experimental study on several experimental regimes~\citep{hooker2019benchmark,ying2019:gnnexplainer,sanchez2020evaluating}, three real-world datasets~\citep{yang2016revisiting} and four different GNN architectures~\citep{kipf2017semi, velickovic2018graph, klicpera2018predict, xu2018how} to evaluate the effectiveness of our explanations. In addition to measuring validity, sparsity, and \fidelity{}, we also compare our approach with respect to the evaluation regime proposed in \gnnexp{}, the faithfulness measure proposed in \cite{sanchez2020evaluating} and the ROAR methodology from \citep{hooker2019benchmark}.

Firstly, we establish that \approach{} outperforms all other baselines over different evaluation regimes on both real-world and synthetic datasets.
Based on \approach{}'s explanation, we retrieve valuable insights into the GNN's behavior: different GNN's derive their decisions from different large portions of the input, and more available features do not mean more relevant features; the GNN's base their classification on different scales on the local homophily; multiple disjoint explanations are possible, i.e., GNN's classification is derived from disjoint parts of the network (duplicated information flow).

To sum up, our main contributions are:
\begin{itemize} %
    \item We theoretically investigate the key properties of \textit{validity}, \textit{sparsity}, and \textit{stability} that a GNN explanation should follow. 
    \item We introduce a novel evaluation metric, \textit{\fidelity{}} derived from principles of \textit{rate-distortion theory} that reflects these desiderata into a single measure.
    \item We propose a simple combinatorial called \approach{} to find high \fidelity{} explanations with theoretically bounded stability. We \textbf{release our code} at \linebreak \url{https://github.com/funket/zorro}.
    \item We perform extensive scale experiments on synthetic and real-world datasets.
    We show that \approach{} not only outperforms baselines for \fidelity{} but also for several evaluation regimes so far proposed in the literature.
\end{itemize}
\section{Related Work}
\label{sec:relatedwork}

Representation learning approaches on graphs encode graph structure with or without node features into low-dimensional vector representations, using deep learning and nonlinear dimensionality reduction techniques.  These representations are trained in an unsupervised  \citep{perozzi2014deepwalk,khosla19comparative,Funke2020Low-dimensional} or semi-supervised manner by using neighborhood aggregation strategies and task-based objectives \citep{kipf2017semi,velickovic2018graph}.

\subsection{Explainability in Machine Learning.} Post-hoc approaches to model explainabiliy are popularized by \textit{feature attribution} methods that aim to assign importance to input features given a prediction either agnostic to the model parameters~\citep{ribeiro2018:anchors, ribeiro2016:lime} or using model specific attribution approaches~\citep{captioningxu2015:attention:nlp,binder2016:lrp,sundararajan2017:integratedgrad}.
\textit{ Instance-wise feature selection} (IFS) approaches \citep{chen2018:l2x,carter2019made,yoon2018invase}, on the other hand, focuses on finding a \textit{sufficient} feature subset or explanation that leads to little or no degradation of the prediction accuracy when other features are masked.
Applying these works directly for graph models is infeasible due to the complex form of explanation, which should consider the complex association among nodes and input features.

\subsection{Explainability in GNNs.}
Explainability approaches for explaining decisions on  a node level include soft-masking approaches \cite{ying2019:gnnexplainer,luo2020parameterized,schnake2020higherorder,lin2021generative,gao2021gnes,schlichtkrull2020interpreting}, Shapely based approaches~\cite{duval2021graphsvx,yuan2021explainability}, surrogate model based methods \cite{vu2020pgm,huang2020graphlime}, and gradients based methods~\cite{pope2019explainability,sanchez2020evaluating,kasanishi2021edge}. 
Soft-masking approaches like \gnnexp{} \cite{ying2019:gnnexplainer} learns a real-valued edge and feature mask such that the mutual information with GNN's predictions is maximized.  
 
An example of a surrogate model based method is PGMExplainer \cite{vu2020pgm} which builds a simpler interpretable Bayesian network  explaining the GNN prediction. 
Others adopt existing explanations approaches such as Shapely~\cite{duval2021graphsvx,yuan2021explainability}, layer-wise relevance propagation~\cite{schnake2020higherorder}, causal effects~\cite{lin2021generative} or LIME~\cite{huang2020graphlime,kasanishi2021edge}, to graph data.
 
The key idea in gradient based methods is to use the gradients or hidden feature map values to approximate input importance. This approach is the most straightforward solution to explain deep models and is quite popular for image and text data.
For graph data, \cite{pope2019explainability} and \cite{sanchez2020evaluating} applied gradient based methods for explaining GNNs, which rely on propagating gradients/relevance from the output to the original model's input.

Another line of work that focuses on explaining decisions at a graph level includes XGNN \cite{XGNN2020} and GNES~\cite{gao2021gnes}. XGNN proposed a reinforcement learning-based graph generation approach to generate explanations for the predicted class for a graph. 
GNES jointly optimizes task prediction and model explanation by enforcing graph regularization and weak supervision on model explanations.
Other works \citep{kang2019explaine,IdahlKA19} focus on explaining unsupervised network representations, which is out of scope for the current work or are specific to the combination of GNNs and NLP~\cite{schlichtkrull2020interpreting}.  

Most of the existing approaches for explaining GNNs are based on soft-masking methods~\cite{pope2019explainability,ying2019:gnnexplainer,luo2020parameterized, schnake2020higherorder, lin2021generative,gao2021gnes}. 
However, soft masks are typically hard for humans to interpret than hard masks due to their low sparsity and inherent uncertainty~\cite{baan2019transformer, pruthi2019learning,kindermans2019reliability,zhang2019dissonance}.
Only a few hard-masking approaches for GNNs exist. 
PGMExplainer~\cite{vu2020pgm} defines explanation in terms of relevant neighborhood nodes influencing the model decision and does not consider node features. 
PGExplainer~\cite{luo2020parameterized} employs a parameterized model to generate soft edge masks with node representations (extracted from target GNN) as input. Unlike our approach, PGExplainer is not model agnostic. Like PGMExplainer, it also does not generate a feature-based explanation. 
SubgraphX~\cite{yuan2021explainability} optimizes for Shapely values based on a Monte Carlo tree search.

\section{Properties of GNN Explanations}
\label{ref:formal}

\subsection{Background on GNNs} Let $G=(V,E)$ be a graph where each node is associated with $d$-dimensional input feature vector. 
GNNs compute node representations by recursive aggregation and transformation of feature representations of their neighbors, which are finally used for label prediction. 
Formally for a $L$-layer GNN, let $\boldsymbol{x}_{n}^{(\ell)}$ denote the feature representation of node $n\in V$ at a layer $\ell \in L$ and ${\mathcal{N}}_{n}$ denotes the set of its direct neighbors. $\boldsymbol{x}_{n}^{(0)}$ corresponds to the input feature vector of $n$.
The $\ell$-th layer of a GNN can then be described as an aggregation of node features from the previous layer followed by a transformation operation. 
\begin{align}
\label{eq:agg}
    \boldsymbol{z}_{n}^{(\ell)}=&\operatorname{AGG}^{(\ell)}\left(\left\{\boldsymbol{x}_{j}^{(\ell-1)} \mid j \in{\mathcal{N}}_{n}\cup\{n\}\right\}\right) 
    \\
    \label{eq:trans}
    \boldsymbol{x}_{n}^{(\ell)}= & \operatorname{TRANSFORMATION} ^{(\ell)}\left(\boldsymbol{z}_{n}^{(\ell)}\right)
\end{align}
Each GNN defines its aggregation function, which is differentiable and usually a permutation invariant function. 
The transformation operation is usually a non-linear transformation, such as employing ReLU non-linear activation.
The final node's embedding ${z}_{n}^{(L)}$ is then used to make the predictions
 \begin{align}\label{eq:pred}
     \Phi(n)\leftarrow \operatorname{argmax}\operatorname{\sigma}(\boldsymbol{z}_{n}^{(L)}\mathbf{W}),
 \end{align}
where $\sigma$ is a sigmoid or softmax function depending on whether the node belongs to multiple or a single class. 
$\mathbf{W}$ is a learnable weight matrix. 
The $i$-th element of $\boldsymbol{z}_{n}^{(L)}\mathbf{W}$ corresponds to the (predicted) probability that node $n$ is assigned to some class $i$. 

\subsection{Defining GNN Explanation, Validity and Sparsity}
\label{sec:valsparse}
\begin{figure}
    \centering
    \includegraphics[width=\linewidth]{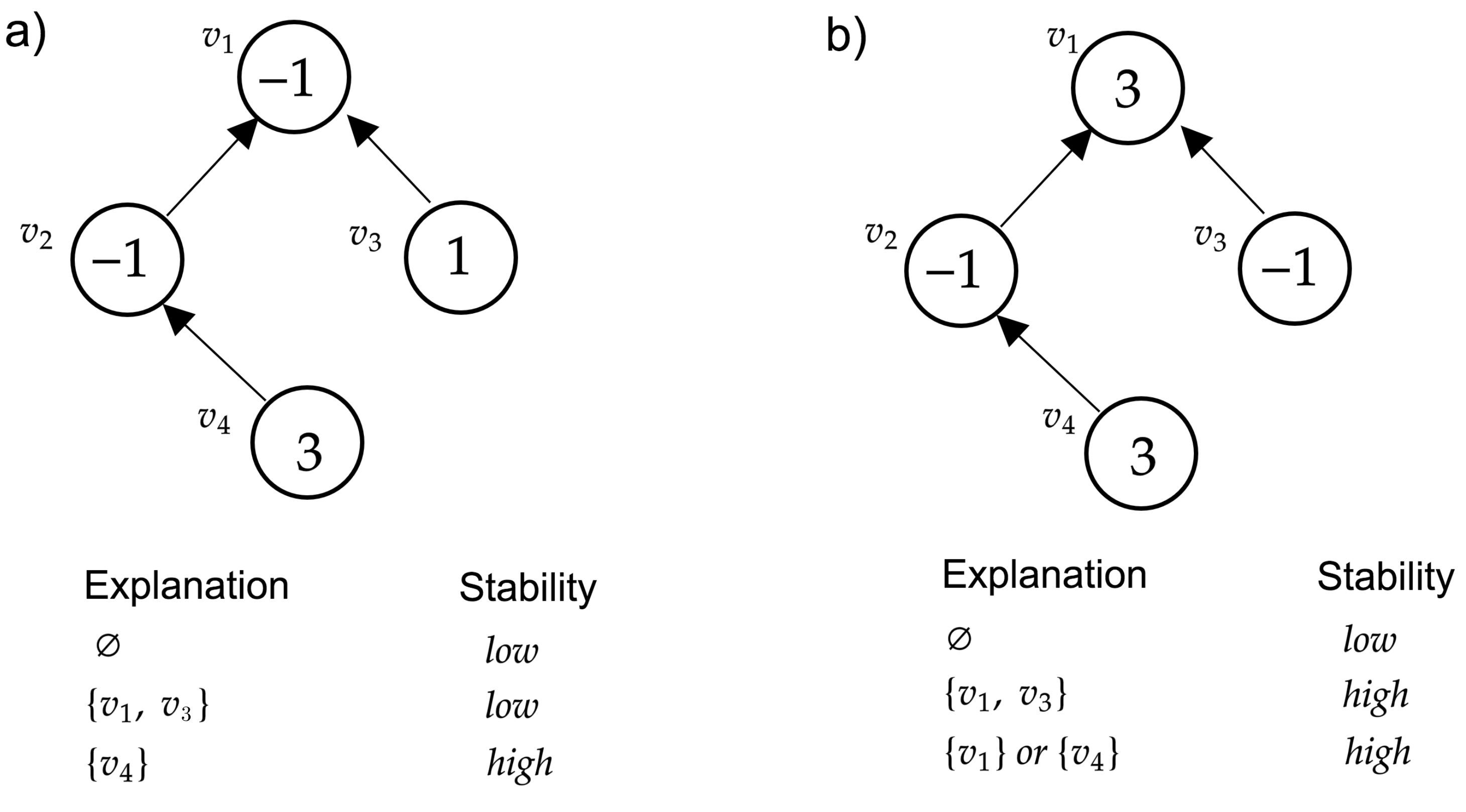}
    \caption{In this synthetic example, we approximate the (node) classification of $v_1$ by GNN with a rule based on the sum of the node features $\sum f(v_i)$.
    All given explanations are valid (when the unselected input is set to 0) and sparse. However, we see that in a) the explanation $\{v_1,v_3\}$ has the same stability as the trivial mask. Example b) highlights that selecting additional elements may not decrease the stability and that even two disjoint explanations are possible. 
    }
    \label{fig:toy_example}
\end{figure}

We are interested in explaining the prediction of the GNN $\Phi(n)$ for any node $n$.
We note that for a particular node $n$, the subgraph taking part in the computation of neighborhood aggregation operation, see Eq.~(\ref{eq:agg}), fully determines the information used by GNN to predict its class. In particular, for a $L$-layer GNN, this subgraph would be the graph induced on nodes in the $L$-hop neighborhood of $n$. For brevity, we call this subgraph the \textit{computational graph} of the query node. We want to point out that the term "computational graph" should not be confused with the neural network's computational graph.

Let $G(n)\subseteq G$  denote the computational graph of the node $n$. Let $X(n)$, or briefly $X$ denotes the feature matrix restricted to the nodes of $G(n)$, where each row corresponds to a $d$-dimensional feature vector of the corresponding node in the computational graph. We define explanation $\mathcal{S}=\{F_s, V_s\}$ as a subset of input features and nodes. In principle, $\mathcal{S}$ would correspond to the feature matrix restricted to features in $F_s$ of nodes in $V_s$.
We quantify the validity and sparsity of $\mathcal{S}$ as follows. 

\begin{definition}[Validity]
The validity score of explanation $\mathcal{S}$ is $1$ if $\Phi(\mathcal{S})=\Phi(X)$ and $0$ otherwise. 
\label{def:validity}
\end{definition}

In literature, the validity of an explanation is usually computed with respect to the baseline $0$, i.e., we set values of all features not in $\mathcal{S}$ to 0. An alternative is to use the average of the feature scores instead. 
As we discuss in Appendix~\ref{sec:app_further_metrics}, our validity is related to one of the metrics from~\citep{yuan2020explainability,singh2021valid,singh2021extracting}.

\begin{definition}[Sparsity]
The sparsity of an explanation is measured as the ratio of bits required to encode an explanation to those required to encode the input. We use explanation entropy to compare sparsity for a fixed input and call this the effective explanation size.
\label{def:sparsity}
\end{definition}
In contrast to other sparsity definitions, such as in~\citep{yuan2020explainability} our definition of sparsity is more general. It can be directly applied for both hard-masks and soft-masks without the need for any transformation.
Without loss of generality, we can assume that an explanation is a continuous mask over the set of features and nodes/edges where the mask value quantifies the importance of the corresponding element. We state the upper bound of the sparsity value in the following proposition.
\begin{proposition}
\label{prop:entropy}
Let $p$ be the normalized distribution of explanation (feature) masks. Then sparsity of an explanation is given by the entropy $H(p)$ and is bounded from above by $\log(|M|)$ where $M$ corresponds to a complete set of features or nodes. 
\end{proposition}
\begin{proof}
 We first compute the normalized feature or node mask distribution, $p(f)$ for $f\in M$.
 In particular, denoting the mask value of $f$ by $\operatorname{mask}(f)$, we have
 $$p(f)= {\operatorname{mask}(f)\over \sum_{f'\in M} \operatorname{mask}(f')  }$$
 Then
$H(p)= -\sum_{f\in M} p(f) \log p(f)$ which achieves its maximum for the uniform distribution, i.e., $p(f)= \frac{1}{|M|}$.
\end{proof}

\subsection{Limitations of Validity and Sparsity}
We illustrate the limitations of previous works, which are based on maximizing validity and sparsity of explanations by a simple example shown in Figure~\ref{fig:toy_example}. 
The example is inspired by the example for text analysis from \cite{camburu2020struggles}.

The input is a graph with node set $V=\{v_1,v_2,v_3,v_4\}$. Each node has a single feature, with the value given in the Figure. Let us assume that these feature values lie in the range from $-2$ to $3$. For any node $v$, we define the model output in terms of simple sum (aggregation) of feature values, $f(v)$ of itself and its two hop-neighborhood. For example,
\[ \Phi(v_1)= \begin{cases} 1, \text{ if } f(v_1) +  f(v_2) +  f(v_3) + f(v_4) \ge 0,\\
0, \text{ otherwise}.
\end{cases}\]
 Now we wish to explain the prediction $\Phi(v_1)=1$. 
 
Consider an explanation $\{v_1,v_3\}$. Clearly it is valid explanation with a validity score of 1, if we set the not selected nodes' features to 0. But if we set $f(v_4)=-2, f(v_2)=0$, the explanation $\{v_1,v_3\}$ is no longer valid.

Similarly, the empty explanation, $\mathcal{S}=\emptyset$, is the sparsest possible explanation which has a validity score of $1$ when all feature values are set to $0$. However, for a different realization of the unimportant features, say for $f(v_2)=-1$ and rest all are set to $0$, the validity score is reduced to $0$. 

We want to emphasize that a particular explanation $\{v_1,v_3\}$ or an empty explanation might be valid for an individual configuration of the features of not selected vertices but not for others. However, a proper explanation should explain the model's prediction independent of the remaining input configuration.

This subtle point is usually ignored by existing explainability approaches, which only evaluate an explanation for a specific baseline of the irrelevant part of the input. In contrast, we propose stability, which takes into account the variance of the validity of an explanation over different configurations of the input's unselected parts.

\begin{definition}[Stability] Let $\mathcal{\mathcal{Y}}$ be a random variable sampled from the  distribution over validity scores for different realizations of $X\setminus \mathcal{S}$. Let $\operatorname{Var}(\mathcal{Y})$ denote the variance of $\mathcal{\mathcal{Y}}$. We define stability $\gamma(\mathcal{S})$ of an explanation, $\mathcal{S}$ as
$$ \gamma(\mathcal{S})= {1\over 1+ \operatorname{Var}(\mathcal{Y})}.$$
\end{definition}

Note that $\gamma(\mathcal{S})\in (0,1]$ holds and achieves, on the one hand, maximum value of $1$ if $\operatorname{Var}(\mathcal{Y})=0$, i.e., when the explanation is completely independent of components in $X\setminus \mathcal{S}$. 
On the other hand, the stability will also be equal to $1$ if the validity of an explanation for all realizations is equal to zero. 
Mathematically, we need to ensure a high expected value of $\mathcal{Y}$ in addition to its low variance. Therefore, we need another metric along with stability. 

To account for the stability of explanations, we introduce a novel metric called \textit{\fidelity{}} which has a sound theoretical grounding in the area of \textit{rate-distortion theory} \cite{sims2016rate}. We describe \fidelity{} and its relation to rate-distortion theory and stability in the next section. %
\section{Rate-Distortion Theory and \fidelity{}}
Rate-distortion theory addresses the problem of determining the minimum number of bits per symbol (also referred to as \textit{rate}) that should be communicated over a channel so that the source signal can be approximately reconstructed at the receiver without exceeding an expected distortion, $D$.
Mathematically we are interested in finding the conditional probability density function, ${Q(\mathcal{S}|X)}$, of the compressed signal or explanation $S$ given the input X such that the expected distortion $D(S,X)$ is upper bounded.

\begin{equation}
\label{eq:rdt}
    \inf_{Q(\mathcal{S}|X)} I_Q(\mathcal{S},X) \text{ such that } \mathbb{E}_{Q}({D(\mathcal{S},X)}) \le D^*,
\end{equation} 
where $I(\mathcal{S},X)$ denotes the mutual information between input $X$ and compressed signal $S$ and $D^*$  corresponds to maximum allowed distortion. Note that Eq.~(\ref{eq:rdt}) requires minimization of mutual information between $X$ and $\mathcal{S}$. Mutual information will be minimized when $\mathcal{S}$ is completely independent of $X$. 

In our explanation framework, the compressed signal $\mathcal{S}$ corresponds to an explanation. 
The effect of minimizing the mutual information between the compressed signal and the input, see Eq.~(\ref{eq:rdt}), would amount to minimize the size of $\mathcal{S}$. A trivial solution of the empty set is avoided by restricting the average distortion of $\mathcal{S}$ in the second part of the objective. 
In particular, compressed signal (explanation) should be such that knowing only about the input on $\mathcal{S}$ and filling in the rest of the information randomly will almost surely preserve the desired output signal (or class prediction).
 
In particular, for graph models, the explanation $\mathcal{S}$ which is a subset of input nodes as well as input features, is most relevant for a classification decision if the expected classifier score remains nearly the same when randomizing the remaining input $X\setminus \mathcal{S}$.

More precisely, we formulate the task of explaining the model prediction for a node $n$, as finding a partition of the components of its computational graph into a subset, $\mathcal{S}$ of relevant nodes and features, and its complement $\mathcal{S}^c$ of non-relevant components. 
In particular, the subset $\mathcal{S}$ should be such that fixing its value to the true values already determines the model output for almost all possible assignments to the non-relevant subset $\mathcal{S}^c$.  The subset $\mathcal{S}$ is then returned as an explanation. As it is a rate-distortion framework, we are interested in an explanation (compressed signal) with the maximum agreement (minimum distortion) with the actual model's prediction on complete input. This agreement, what we refer to as \textit{\fidelity{}} is quantified by the expected validity score of an explanation over all possible configurations of the complement set $\mathcal{S}^c$.

\subsection{\fidelity{}}
To formally define \fidelity{}, let us denote with $Y_{\mathcal{S}}$ the perturbed feature matrix obtained by fixing the components of the ${\mathcal{S}}$ to their actual values and otherwise noisy entries. 
The values of components in $\mathcal{S}^c$ are then drawn from some noisy distribution, $\mathcal{N}$.
Let $\mathcal{S}=\{V_s, F_s\}$ be the explanation with selected nodes $V_s$ and selected features $F_s$.

Let $M(\mathcal{S})$, or briefly $M$, be the mask matrix such that each element $M_{i,j}=1$ if and only if $i$th node (in $G(n)$) and $j$th feature are included in sets $V_s$ and $F_s$ respectively and $0$ otherwise.
Then the perturbed input is given by
\begin{equation}
 Y_{\mathcal{S}} = X\odot M(\mathcal{S}) + Z\odot(\mathbbm{1} - M(\mathcal{S})), Z\sim \mathcal{N},
    \label{eq:noisy_features_matrix}
\end{equation}
where $\odot$ denotes an element-wise multiplication, and $\mathbbm{1}$ a matrix of ones with the corresponding size. 

\begin{definition}[\fidelity{}]
The \textit{\fidelity{}} of explanation $\mathcal{S}$ with respect to the GNN $\Phi$ and the noise distribution $\mathcal{N}$ is given by
\begin{equation}
\label{eq:fidelity}
\mathcal{F}(\mathcal{S}) = \mathbb{E}_{Y_{\mathcal{S}}|Z\sim \mathcal{N}} \left[\mathbbm{1}_{\Phi\left(X\right)=\Phi(Y_{\mathcal{S}})}\right].
\end{equation}
In simple words, \fidelity{} is computed as the expected validity of the perturbed input $Y_\mathcal{S}$. 
\label{def:fidelity}
\end{definition}
Note that high \fidelity{} explanations would be stable by definition, i.e., their validity score would not vary significantly across different realizations of $\mathcal{S}^c$.

\begin{theorem}
\label{thm:fidelity}
An explanation with \fidelity{} $p$ has stability value of ${1\over 1+  p(1-p)}$.
\end{theorem}
\begin{proof}
Let $\mathcal{Y}$ be the random variable corresponding to validity score for an explanation $\mathcal{S}$. Note that 
$$\mathbb{E}(\mathcal{Y}) = \mathbb{E}_{Y_{\mathcal{S}}|Z\sim \mathcal{N}} \left[\mathbbm{1}_{\Phi\left(X\right)=\Phi(Y_{\mathcal{S}})}\right]= p$$
where $Y_S$ and $Z$ are as defined in \eqref{eq:noisy_features_matrix}.

Note that $\mathcal{Y}$ can be understood  as a sample drawn from a Bernoulli distribution with a mean equal to \fidelity{} value, i.e., 
$\mathcal{Y} \sim \operatorname{Ber}(p)$
The variance of a Bernoulli distributed variable $\mathcal{Y}$ is given by $ p(1-p)$. The proof is completed then by substituting the variance in the definition of stability.
\end{proof}
Theorem~\ref{thm:fidelity} implies that for \fidelity{} greater than $0.5$, the stability increases with an increase in \fidelity{} and achieves a maximum value of $1$ when \fidelity{} reaches its maximum value of $1$. Therefore, to ensure high stability, it suffices to find high \fidelity{} explanations. As stability is theoretically bounded with respect to \fidelity{}, we do not additionally report stability in our experiments. %

\section{Maximizing \fidelity{}}
We propose a simple but effective greedy combinatorial approach, which we call \approach{}, to find high \fidelity{} explanations.  
By fixing the \fidelity{} to a certain user-defined threshold, say $\tau$, we are interested in the sparsest explanation, which has a \fidelity{} of at least $\tau$.

The pseudocode is provided in Algorithm~\ref{alg:recursive}. 
Let for any node $n$, $V_n$ denote the vertices in its computational graph $G(n)$, i.e., the set of vertices in $L$-hop neighborhood of node $n$ for an L-layer GNN; and $F$ denote the complete set of features. 
We start with zero-sized explanations and select as first element 
\begin{equation}
\label{eq:init}
\underset{f\in F}{\mathrm{argmax}} ~~\mathcal{F}(V_n, \{f\}) \quad \text{or} \quad \underset{v\in V_n}{\mathrm{argmax}} ~~\mathcal{F}(\{v\}, F),
\end{equation}
whichever yields the highest \fidelity{} value. We iteratively add new features or nodes to the explanation such that the \fidelity{} is maximized over all evaluated choices. 
Let $V_p$ and $F_p$ respectively denote the set of possible candidate nodes and features that can be included in an explanation at any iteration. 
We save for each possible node $v\in V_p$ and feature $f\in F_p$ the ordering $R_{V_p}$ and $R_{F_p}$ given by the \fidelity{} values $\mathcal{F}(\{v\}, F_p)$ and  $\mathcal{F}(V_p, \{f\})$ respectively.
To reduce the computational cost, we only evaluate each iteration the top $K$ remaining nodes and features determined by $R_{V_p}$ and $R_{F_p}$.

\begin{algorithm}[h]
    \caption{\approach{}$(n, \tau)$}
\label{alg:recursive}
\textbf{Input:} node $n$, threshold $\tau$ \\
\textbf{Output:} explanation, i.e. node mask $V_s$ \& feature mask $F_s$
    \begin{algorithmic}[1]
\State $V_n\leftarrow$ set of vertices in $G(n)$
\State $F_p \leftarrow$ set of node features
\State $V_r=V_p$, $F_r=F_p$, $V_s=\emptyset$, $F_s=\emptyset$
\State {\small{$R_{V_p}\leftarrow $ list of $v\in V_p$ sorted by $\mathcal{F}(\{v\}, F_p)$ }} \label{alg:init_1}
\State $R_{F_p}\leftarrow$ list of $f\in F_p$ sorted by $\mathcal{F}(V_p,\{f\})$ \label{alg:init_2}
\State Add maximal element to $V_s$ or $F_s$ as in (\ref{eq:init})
\While{$\mathcal{F}(V_s,F_s) \le \tau $}
\State $\Tilde{V}_s = V_s \cup \underset{v\in \operatorname{top}_K(V_r)}{\mathrm{argmax}} \mathcal{F}(\{v\}\cup V_s, F_s)$
\State $\Tilde{F}_s= F_s \cup \underset{f\in \operatorname{top}_K(F_r)}{\mathrm{argmax}} \! \mathcal{F}(V_s, \{f\}\cup F_s)$
\If{$\mathcal{F}(\Tilde{V}_s,F_s) \le \mathcal{F}(V_s,\Tilde{F}_s)$}
\State $F_r=F_r\setminus \{f\}$, $F_s=\Tilde{F}_s$
\Else
\State $V_r=V_r\setminus \{v\}$, $V_s= \Tilde{V_s}$
\EndIf
\EndWhile
\State \textbf{return} $\{V_s, F_s\}$ 
\end{algorithmic}
\end{algorithm}

As shown in Fig.~\ref{fig:toy_example}, an instance can have multiple valid, sparse, and stable explanations.
Therefore, we also propose a variant of \approach{}, which continues searching for further explanations:
Once we found an explanation with the desired \fidelity{}, we discard the chosen elements from the feature matrix $X$, i.e., we never consider them again as possible choices in computing the next explanation. 
We repeat the process by finding relevant selections disjoint from the ones already found. 
To ensure that disjoint elements of the feature matrix $X$ are selected, we recursively call Algorithm~\ref{alg:recursive} with either remaining (not yet selected in any explanation) set of nodes or features.
Finally, we return the set of explanations such that the \fidelity{} of $\tau$ cannot be reached by using all the remaining components that are not in any explanation.
For a detailed explanation of the details and the reasoning behind various design choices, we refer to Appendix~\ref{sec:algorithm_details}.

The pseudocode to compute \fidelity{} is provided in Algorithm~\ref{alg:fid}. Specifically we generate the obfuscated instance for a given explanation $\mathcal{S}=\{V_s,F_s\}$, $Y_{\mathcal{S}}$ by setting the feature values for selected node-set $V_s$ corresponding to selected features in $F_s$ to their true values. 
To set the irrelevant values, we randomly choose a value from the set of all possible values for that particular feature in the dataset $\mathcal{X}$. 
To approximate the expected value in Eq.~(\ref{eq:fidelity}), we generate a finite number of samples of $Y_{\mathcal{S}}$.
We then compute \fidelity{} as average validity with respect to these different baselines.

\begin{algorithm}[h]
   \caption{$\mathcal{F}(V_s,F_s)$}
\label{alg:fid}
\textbf{Input:} node mask $V_s$, feature mask $F_s$ \\
\textbf{Output:} \fidelity{} for the given masks
\begin{algorithmic}[1]
\For{$i=0, \ldots, \text{samples}$} 
\State Set $Y_{\{V_s, F_s\}}$, i.e.\ fix the selected values and otherwise retrieve random values from the respective columns of $\mathcal{X}$
\If{$\Phi(Y_{\{V_s, F_s\}})$ matches the original prediction of the model}
\State correct$+=1$
\EndIf
\EndFor
\State \textbf{return} $\frac{\text{correct}}{\text{samples}}$
\end{algorithmic}
\end{algorithm}

\begin{theorem}\label{thm:zorro}
\approach{} has the following properties.
\begin{enumerate}
    \item \approach{} retrieves explanation with at least \fidelity{} $\tau$: %
    \[
    \mathcal{F}(V_s, F_s) \ge \tau.
    \]
    \item The runtime of \approach{} is independent of the size of the graph. The runtime complexity of \approach{} for retrieving an explanation is given by $$O(t\cdot\max(|V_n|,|F|)),$$
    where $t$ is the run time of the forward pass of the GNN $\Phi$.
    \item For any retrieved explanation $\mathcal{S}$ and $\tau\ge 0.5$, the stability score is
    $\gamma(\mathcal{S})\ge {1\over 1 +\tau(1-\tau)}$.
    
\end{enumerate}
\end{theorem}
For the proof and the discussion on the choice of noise distribution , we refer to appendix~\ref{sec:proof_thm_zorro}.

\section{Experimental Setup}
\label{sec:setup}

The evaluation of post-hoc explanation techniques has always been tricky due to the lack of ground truth.
Specifically, for a model prediction, collecting the ground truth explanation is akin to asking the trained model what it was thinking about -- an impossibility and hence a dilemma.
There is no clear solution to the ground-truth dilemma. However, previous research has attempted varying experimental regimes, each with its simplifying assumptions. 
We conduct a comprehensive set of experiments adopting the three dominant existing experimental regimes from the literature --   \textit{real-world graphs with unknown ground truth, remove and retrain, and synthetic graphs with known ground truth}. 
Later on, we will reflect on the limitations of their assumptions and the threats they might pose to our results' validity.

\subsection{Evaluation without Ground Truth}
\label{sec:evalwithoutGT}

In the absence of ground truth for explanations, we can still evaluate posthoc explanations using the desirable properties of the explanations introduced by us, i.e., sparsity, stability (quantified via \fidelity{}), and validity:

\begin{RQ}
\label{rq:fidelity_stability_sparsity}
How effective is \approach{} as compared to existing methods in terms of sparsity, \fidelity{}, and validity? 
\end{RQ}
Note that these metrics are not always correlated.  For example, an explanation can have a high validity score but low stability or \fidelity{}. In the following, we describe the real-world datasets that we use to compare explanations.

\mpara{Datasets and GNN Models.} We use the most commonly used datasets \textbf{Cora}, \textbf{CiteSeer} and \textbf{PubMed}~\citep{yang2016revisiting}.
We evaluate our approach on four different two-layer graph neural networks: \textbf{GCN}, graph attention network (\textbf{GAT})~\citep{velickovic2018graph}, the approximation of personalized propagation of neural predictions (\textbf{APPNP})~\citep{klicpera2018predict}, and graph isomorphism network (\textbf{GIN})~\citep{xu2018how}.
We evaluate these combinations with respect to \textit{validity}, \textit{sparsity}, and \textit{\fidelity{}} for $300$ randomly selected query nodes. 
To calculate node sparsity for those approaches which retrieve soft edge masks, such as \gnnexp{}, we follow~\cite{sanchez2020evaluating} and create node masks by distributing the edge mask value equally onto the endpoint of the respective edges. For example, if a particular edge $(u,v)$ the corresponding edge mask has a value of 0.5, then nodes $u$ and $v$ would be given a node mask of $0.25$ each.

In Appendix~\ref{sec:further_experiments}, we provide additional experimental results in which we investigate: i) the effect of the number of samples used for calculating the \fidelity{} in \approach{}, ii) further variations of the \fidelity{} threshold, iii) explanations using four additional metrics proposed by \citep{yuan2020explainability}, and iv) the impact of larger computational graphs on explanation approaches by using the Amazon Computers dataset~\citep{shchur2018pitfalls}. 

\subsection{Remove and Retrain}
\label{sec:roar}  

In this experimental regime, we follow the remove-and-retrain (or ROAR) paradigm of evaluating explanations~\cite{hooker2019benchmark} that is based on retraining a neural network based on the explanation outputs. 
ROAR removes the fraction of input features deemed to be the most important according to each explainer and measures the change to the model accuracy upon retraining. 
Thus, the most accurate explainer will identify inputs as necessary whose removal causes the most damage to model performance relative to all other explainers.
Note that, unlike the other evaluation schemes, firstly, ROAR is a global approach in that it forces a fixed set of features to be removed.
Secondly, ROAR involves retraining the model, whereas other approaches have interventions purely on the outputs of the trained model.

\begin{RQ}
\label{rq:globalrelevence}
How effective is \approach{} when its output explanations are used for retraining a new GNN model?
\end{RQ}

\subsection{Evaluation with Ground Truth}

Although it is hard to obtain ground-truth data from real-world datasets, previous works have constructed synthetic datasets with known subgraph structures that GNN models learn to predict the output label~\cite{ying2019:gnnexplainer}.
We consider the only synthetic dataset proposed in~\cite{ying2019:gnnexplainer} having features called BA-Community. 
First, we create a community with a base Barab\'{a}si-Albert (BA) graph and attach a five-node house graph to randomly selected nodes of the base graph. 
Nodes are assigned to one of the eight classes based on their structural roles and community memberships. For example, there are three functions in a house-structured motif: the house's top, middle, and bottom nodes. Following~\cite{ying2019:gnnexplainer}, only the class assignments of the house nodes have to be explained, and the respective house is regarded as ``explanation ground truth''.

\mpara{Node features for synthetic graphs.} Nodes have normally distributed feature vectors.
Each node has eight feature values drawn from $\mathcal{N}(0,1)$ and two features drawn from $\mathcal{N}(-1,.5)$ for nodes of the first community or $\mathcal{N}(1,.5)$ otherwise.  The feature values are normalized within each community, and within each community, $0.01\ \%$ of the edges are randomly perturbed. 
Note that for reproducibility, we strictly follow the published implementation of GNNExplainer.
For the known ground-truth regime, we are interested in answering the following research question:

\begin{RQ}
\label{rq:groundtrutheval}
Are Zorro's explanations \textbf{accurate}, \textbf{precise} and \textbf{faithful} to available ground truth explanations?
\end{RQ}

\subsubsection{Metrics} 

To compare against the known ground truth, we use various metrics proposed earlier in literature for synthetic datasets -- \textit{accuracy, precision, faithfulness}.

\mpara{Accuracy} measures the fraction of correctly classified nodes in the explanation.
Note that only reporting accuracy as a metric does not portray the complete picture. For example, in our imbalanced dataset of five positive nodes (the house motif), out of 100 other nodes in the computational graph, high accuracy can be achieved by a trivial selection of five neighbors or sometimes even none.
Therefore, we also report the {precision} value, which emphasizes the fraction of correct predictions:

\mpara{Precision} is defined as the fraction of returned nodes that are also in the explanation set. Precision provides more reliable results than accuracy when the input is much larger than the explanation. To compute accuracy and precision for baselines, we transform the baselines' results into a node mask of the five most important nodes, which is the size of the explanation ground truth. 

\mpara{Faithfulness.} Faithfulness is based on the assumption that a more accurate GNN leads to more precise explanations~\cite{sanchez2020evaluating}. 
We measure faithfulness by comparing the explanation performance of the GNN model at an intermediate training epoch against the fully trained GNN (final model trained until convergence).
Specifically, we generate two ranked lists corresponding to test accuracy and precision of retrieved explanations at different epochs. We then compute faithfulness as the rank correlation of these two lists measured using Kendall's tau $\tau_{\text{Kendall}}$.

\subsection{Baselines and Competitors}

For a comprehensive quantitative evaluation we chose our baselines from the three different categories of post-hoc explanations models consisting of (i) soft-masking approaches like
\textbf{\gnnexp{}}, which returns a continuous feature and edge mask and \textbf{PGE~\cite{luo2020parameterized}} learns soft masks over edges in the graph (ii) surrogate model based hard-masking approach, \textbf{PGM}~\cite{vu2020pgm}, which returns a binary node mask, iii) Shapely based hard masking approach SubgraphX~\cite{yuan2021explainability}, which returns a subgraph as an explanation, and (iv) gradient-based methods \textbf{Grad \& GradInput}~\cite{shrikumar2017:deeplift} which utilize gradients to compute feature attributions. Specifically, we take the gradient of the rows and columns of the input feature matrix $X$, which corresponds to the features' and nodes' importance. 
For GradInput, we also multiply the result element-wise with the input. In the case of PGM, we use the author's default settings to choose the best node mask.
Besides, we employ an \textbf{empty explanation} as the naive baseline.
We could only run SubgraphX for the small synthetic dataset, due to its long runtime (see Appendix~\ref{sec:experiments}).

\mpara{\approach{} Variants.} For our approach \approach{}, we retrieved explanations for the thresholds $\tau=.85$ and $\tau=.98$ with $K=10$. All \fidelity{} values were calculated based on $100$ samples.

We refer to Appendix~\ref{sec:experiments} and the available implementation for further details of the models and the training of the GNNs. %

\settowidth\rotheadsize{Sparsity)} 
\begin{table*}[htbp]
    \centering
    \caption{Analysis of the average \textit{sparsity (Definition~\ref{def:sparsity}), \fidelity{} (Definition~\ref{def:fidelity}),} and \textit{validity} (Definition~\ref{def:validity}) of the explanations. 
    The smaller the explanation size larger is the sparsity. 
    As stability can be directly derived from \fidelity{} and increases with \fidelity{} $>0.5$ (see Theorem~\ref{thm:fidelity}), it suffices to compare \fidelity{} to ensure stability. PGM and PGE are not included in the feature sparsity because they don't retrieve feature masks. 
    }
    {\small
    \setlength{\tabcolsep}{3.5pt}
    \begin{NiceTabular}{ll ccc ccc ccc ccc}
\toprule
\multirow{2}{*}{\textbf{Metric}} & \multirow{2}{*}{\textbf{Method}} & \multicolumn{4}{c}{\textbf{Cora}} & \multicolumn{4}{c}{\textbf{CiteSeer}} & \multicolumn{4}{c}{\textbf{PubMed}} \\
 &  &   GCN &   GAT &   GIN & APPNP &      GCN &   GAT &   GIN & APPNP &    GCN &   GAT &   GIN & APPNP \\
\midrule
 \multirow{5}{*}{ \rotcell{Features-Sparsity}}& \gnnexp{}   &  7.27 &  7.27 &  7.27 &  7.27 &     8.21 &  8.21 &  8.21 &  8.21 &   6.21 &  6.21 &  6.21 &  6.21 \\
 & Grad           &  4.08 &  4.22 &  4.45 &  4.08 &     4.19 &  4.28 &  4.41 &  4.18 &   4.41 &  4.51 &  4.89 &  4.46 \\
 & GradInput      &  4.07 &  4.25 &  4.37 &  4.08 &     4.17 &  4.29 &  4.33 &  4.17 &   4.41 &  4.51 &  4.92 &  4.47 \\
 & \approach{}  $(\tau = .85)$        &  \bf{1.91} &  \bf{2.29} &  \bf{3.51} &  \bf{2.26} &     \bf{1.81} &  \bf{1.84} &  \bf{3.67} &  \bf{1.97} &   \bf{1.60} &  \bf{1.52} &  \bf{2.38} &  \bf{1.75} \\
 & \approach{}  $(\tau = .98)$      &  2.69 &  3.07 &  4.34 &  3.18 &     2.58 &  2.60 &  4.68 &  2.78 &   2.55 &  2.58 &  3.21 &  2.86 \\
 \midrule
 \multirow{7}{*}{ \rotcell{Node-Sparsity}}& \gnnexp{}   &   2.48 &  2.49 &  2.56 &  2.51 &     1.67 &  1.67 &  1.70 &  1.68 &    2.7 &  2.71 &  2.71 &  2.71 \\
 & PGM    &  2.06 & 1.82 &  \bf{1.66} & 1.99 & 1.47 &  1.59 & \bf{1.10} & 1.54 & 1.64 &  1.16 & \bf{1.62} &  2.93 \\
 & PGE & 1.86 & 1.86 & 1.78 & 1.94 &1.48 & 1.40& 1.36 & 1.41 & 1.91 &  1.81 & 1.85 & 1.92  \\ 
 & Grad            &  2.48 &  2.34 &  2.25 &  2.35 &     1.70 &  1.61 &  1.55 &  1.60 &   2.91 &  2.76 &  3.11 &  2.73 \\
 & GradInput        &  2.53 &  2.43 &  2.23 &  2.41 &     1.61 &  1.58 &  1.54 &  1.52 &   3.02 &  2.94 &  3.41 &  2.81 \\

 & \approach{}  $(\tau = .85)$   &  \bf{1.28} &  \bf{1.30} &  1.90 &  \bf{1.16} &     \bf{1.05} &  \bf{0.92} &  1.36 &  \bf{0.83} &   \bf{1.07} &  \bf{0.87} &  1.77 &  \bf{0.79} \\
 & \approach{}  $(\tau = .98)$  &  1.58 &  1.59 &  2.17 &  1.48 &     1.26 &  1.09 &  1.58 &  1.07 &   1.51 &  1.31 &  2.18 &  1.25 \\
\midrule
 \multirow{7}{*}{ \rotcell{\fidelity{}}}& \gnnexp{}   &  0.71 &  0.66 &  0.52 &  0.65 &     0.68 &  0.69 &  0.51 &  0.62 &   0.67 &  0.73 &  0.67 &  0.72 \\
 & PGM  & 0.84 &  0.77 &  0.60 &  0.89 & 0.92 & 0.93 &  0.73 &  0.95 &   0.78 &  0.69 &  0.74 &  0.96 \\ 
 & PGE & 0.50 & 0.53 & 0.35 & 0.49 &0.64 &0.60 & 0.51 & 0.61 & 0.49 & 0.61  & 0.56 & 0.50  \\ 
 & Grad           &  0.15 &  0.18 &  0.19 &  0.17 &     0.17 &  0.19 &  0.28 &  0.18 &   0.37 &  0.43 &  0.42 &  0.37 \\
 & GradInput      &  0.15 &  0.18 &  0.18 &  0.16 &     0.16 &  0.18 &  0.26 &  0.17 &   0.36 &  0.42 &  0.42 &  0.36 \\
 & Empty Explanation      &  0.15 &  0.18 &  0.18 &  0.16 &     0.16 &  0.18 &  0.26 &  0.17 &   0.36 &  0.42 &  0.42 &  0.36 \\
  & \approach{}  $(\tau = .85)$       &  0.87 &  0.88 &  0.86 &  0.88 &     0.87 &  0.86 &  0.87 &  0.86 &   0.86 &  0.88 &  0.88 &  0.87 \\
 & \approach{}  $(\tau = .98)$    &  \bf{0.97} &  \bf{0.97} &  \bf{0.96} &  \bf{0.97} &     \bf{0.97} &  \bf{0.97} &  \bf{0.97} &  \bf{0.96} &   \bf{0.96} &  \bf{0.97} &  \bf{0.97} & \bf{0.96} \\
 \midrule
\multirow{7}{*}{ \rotcell{Validity}} & \gnnexp{}   &  0.89 &  0.95 &  0.83 &  0.84 &     0.87 &  0.92 &  0.58 &  0.93 &   0.60 &  0.81 &  0.71 &  0.87 \\
 & PGM & 0.89 &  0.90 &  0.64 &  0.94 & 0.95 & 0.95 &  0.76 &  0.97 &   0.86 &  0.80 &  0.62 &  0.97 \\
 & PGE & 0.51 & 0.54 & 0.34 &0.45  &0.62 & 0.59& 0.54  & 0.62 & 0.51 &  0.61 &0.57  & 0.48  \\ 
 & Grad           &  0.26 &  0.25 &  0.15 &  0.18 &     0.28 &  0.25 &  0.12 &  0.26 &   0.36 &  0.49 &  0.50 &  0.38 \\
 & GradInput      &  0.22 &  0.22 &  0.12 &  0.17 &     0.18 &  0.16 &  0.08 &  0.19 &   0.36 &  0.49 &  0.50 &  0.37 \\
 & Empty Explanation      &  0.22 &  0.22 &  0.11 &  0.17 &     0.18 &  0.16 &  0.08 &  0.19 &   0.36 &  0.49 &  0.50 &  0.37 \\
 & \approach{}  $(\tau = .85)$          &  \bf{1.00} &  \bf{1.00} &  0.83 &  \bf{1.00} &     \bf{1.00} &  \bf{1.00} &  0.77 &  \bf{1.00} &   0.90 &  \bf{1.00} &  0.84 &  \bf{1.00} \\
 & \approach{}  $(\tau = .98)$       &  \bf{1.00} &  \bf{1.00} &  \bf{0.90} &  \bf{1.00} &     \bf{1.00} &  \bf{1.00} &  \bf{0.91} &  \bf{1.00} &   \bf{0.98} &  \bf{1.00} &  \bf{0.87} &  \bf{1.00} \\
\bottomrule
\end{NiceTabular}
    }
    \label{tab:real_main}
\end{table*}

\section{Experimental Results}
\label{sec:results}

In presenting our experimental results, we begin with RQ~\ref{rq:fidelity_stability_sparsity} that relates to the regime where we consider real-world datasets but without ground-truth explanations in Section~\ref{sec:realeval}. 
Continuing with the real-world datasets, we will discuss the global impact of explanation approaches when GNN models are retrained based on the explanations in Section~\ref{rq:globalrelevence}.
To our knowledge, we are the first to evaluate GNN-explanation approaches in the retraining setup.
Finally, in Section~\ref{sec:synthetic-data}, we will check the effectiveness of \approach{} on synthetic datasets where ground-truth for explanations is known.

\subsection{Evaluation with Real-World Data}
\label{sec:realeval}

To answer RQ~\ref{rq:fidelity_stability_sparsity}, we evaluate \approach{}'s performance on three standard real-world datasets -- \textbf{Cora}, \textbf{CiteSeer} and \textbf{PubMed}.
As discussed in the last section, real-world datasets do not have accompanying ground-truth explanations.
Instead, the results of our experiments are summarized in Table~\ref{tab:real_main} where we compare the performance of various explanation methods in terms of validity, sparsity, and \fidelity{}.

\mpara{Validity and \fidelity{}.} 
We re-iterate that \fidelity{} measures the stability of the explanations. 
The validity, on the other hand, measures if the explanation alone retains the same class predictions.
We first observe that gradient-based approaches obtain low \fidelity{} and validity compared to other soft-masking baselines like PGM and \gnnexp{}.
We observe that even the empty baseline achieves validity in the same range of $0.11 - 0.50$ as gradient-based methods. 
Hence, selecting no nodes and no features in the explanation, as done in the empty explanation baseline, yields similar performance as the gradient-based explanations. 
This result also establishes the superiority of GNN-specific explanation methods as PGM and \gnnexp{}.
Interestingly, while \gnnexp{} outperforms PGM for Cora in terms of validity, PGM finds overall more stable explanations and shows higher \fidelity{}. On the other side, PGE performs worst out of all masks-based explainers.

Since \approach{} optimizes for \fidelity{}, we expectantly deliver high performance for \fidelity{}. 
However, \approach{} also convincingly outperforms all the existing baseline approaches for validity even if it is not explicitly optimized for validity.
Additionally, a significant result here is that our heuristic yet efficient greedy procedure is already sufficient to produce near-optimal validity and \fidelity{} values.

\mpara{Node and Feature Sparsity.} 
Note that we differentiate between node and feature sparsity because explanation methods like PGM do not produce feature attributions.
Moreover, we report the sparsity as the effective explanation size that is the entropy of the retrieved masks. The larger the explanation size, the lower will be the sparsity.
First, we compare soft-masking approaches, i.e., gradient-based approaches, PGE, and \gnnexp{}.
We observe that the feature sparsity of \gnnexp{}, somewhat surprisingly, is less sparse than even gradient-based approaches. Since PGE and \gnnexp{} return soft edge masks, we compute the corresponding node mask as a sum of the masks of the edges which contain the corresponding node.
In terms of node sparsity, PGE outperforms all other soft-mask-based approaches. As PGE does not produce a feature mask, in other words, it selects all features, feature sparsity is not provided.
A low sparsity for soft-masking approaches implies a near-uniform feature attribution and consequently lower interpretability. 
On the other hand, explanations produced by \approach{} and PGM are boolean masks or hard masks.
Since PGM only retrieves node masks, only a comparison based on node sparsity is possible. 
PGM outperforms with respect to node sparsity all soft masking approaches. 
However, for all cases, but GIN, \approach{} retrieves even sparser node masks.

We see that \approach{} produces significantly sparser explanations in comparison to soft-masking approaches. 
Between the variants of \approach{}, the explanations of \approach{} for $\tau=0.85$ are expectedly lower in sparsity than for $\tau=0.98$ that is a more constrained version of \approach{}. 
However, note that a lower sparsity comes with an advantage of higher \fidelity{} and validity.

\mpara{Key Takeaways.} Our crucial takeaway from this experiment is that \approach{} convincingly outperforms all other explanation methods across all datasets and GNN models.
To answer RQ~\ref{rq:fidelity_stability_sparsity} quantitatively, we report the average improvement of \approach{} $(\tau=0.98)$ (with respect to the best performing baseline for each metric, model, and dataset): \approach{} achieves a reduction of $72\%$ and $94\%$ in the effective node, respectively, feature explanation size; increase by $20\%$ in \fidelity{} and $11\%$ in validity of explanations.

\subsection{Evaluation with Remove and Retraining}
\label{sec:roar-results}
We now present the results that estimate the global relevance of explanations by adapting the ROAR technique as already described earlier in Section \ref{sec:evalwithoutGT}.
In our setup, given (a) a training set, (b) a GNN model that we want to explain and, (c) an explanation method, we retrieved explanations for each node in the training set.
Next, we sum all feature masks corresponding to the retrieved explanations (of all training nodes) and choose the top-$k$ features based on the aggregated value. 
For hard masks, this procedure is equivalent to selecting the top-$k$ most frequently retrieved features. 
Finally, we retrain the GNN model again on the same training set but with the selected top-$k$ features.
Figure~\ref{fig:retraining} reports the performance drop in the model's test accuracy after the remove-and-retrain procedure.
Note that the fewer features are needed to achieve similar test accuracy, the better the explanations' quality.  

We report our results for $k\in\{1, 5, 10, 50, 100\}$ most essential features using GCN as the GNN and over the CORA dataset. 
First, we observe that using only the top-$10$ important features using \approach{} ($\tau=.98$) already achieves a test accuracy of $0.72$ compared to $0.79$ on all $1433$ features. 
Selecting $100$ features however using \approach{} ($\tau=.98$), causes only a minor performance drop of $\Delta<0.01$. 
Similar to Table~\ref{tab:real_main}, Grad achieves slightly better results than GradInput. 
Interestingly, \gnnexp{} performs poorly, and the possible reason for this is its non-sparse feature masks (as seen in the previous section).
Since PGM does not retrieve feature masks, it could not be evaluated in this setting. 

To answer RQ~\ref{rq:globalrelevence}, we find that \approach{} effectively chooses good global features, aggregated from \approach{}'s local explanation, in comparison to other explanation approaches.
Surprisingly, the gradient-based methods outperform GNN-specific \gnnexp{} approaches.
This is possible because we are experimenting with feature masks (and not node masks) and that gradient-based approaches are optimized for non-relational models.

\begin{figure}
    \centering
    \includegraphics[width=0.8\linewidth]{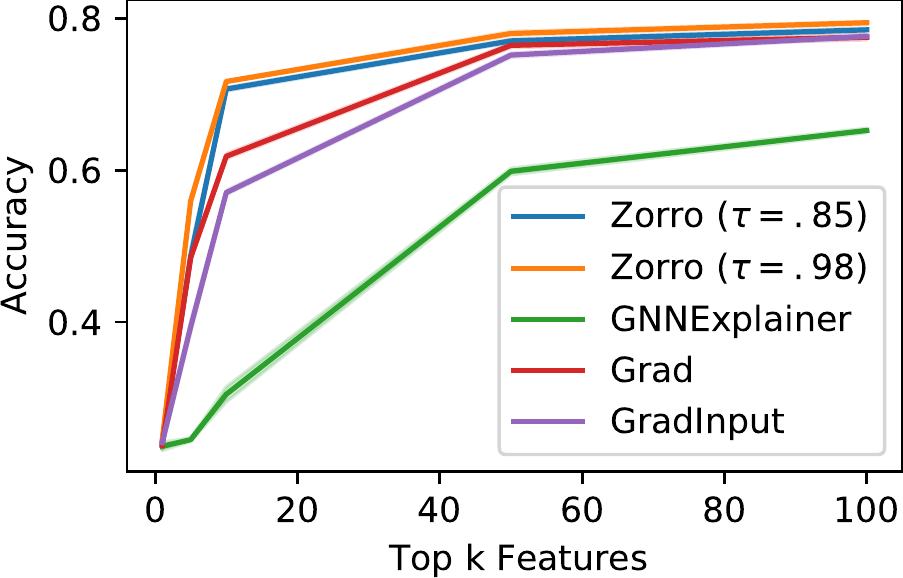}
    \caption{Test accuracy after retraining GCN on Cora based on the top k features. We repeated the retraining 20 times, report the mean, and observed a variation of below $.001$.} %
    \label{fig:retraining}
\end{figure}
\subsection{Evaluation with Ground-Truth}
\label{sec:synthetic-data}

\begin{table}[htbp]
 \setlength{\tabcolsep}{3.5pt}
    \centering
    \caption{Performance of the node explanation on the synthetic dataset. The sparsity is calculated for the retrieved node mask. The high accuracy with empty explanation by large size of negative set. This also points to the pitfall of using Accuracy alone as the measure of evaluating explanations when ground truth is available. }
    \small
    \begin{NiceTabular}{@{}lr ccc@{}}
    \toprule
         \textbf{Method} & \textbf{Prec.} $\uparrow$ & \textbf{Acc.} $\uparrow$& \textbf{Sparsity} $\downarrow$&  \textbf{\fidelity{}} $\uparrow$\\
         \midrule
\gnnexp{}  &            0.40 &           0.81 &               1.68 &      0.63 \\
{PGM}  &            0.75 &           0.93 &   2.30 &      0.81 \\
{PGE}  &     0.19 &  0.21   &  1.73 & 0.58      \\
{SubgraphX} &             0.72 &           0.94 &               1.25 &      0.82 \\
Grad          &            0.87 &           0.95 &               1.61 &      0.70 \\
GradInput     &            0.89 &           \bf{0.96} &               1.61 &      0.56 \\
$\emptyset$ - Explanation  &             0.00 &           0.84 &                0.00 &      0.55 \\
\approach{} $(\tau=.85)$       &            \bf{0.95} &           0.90 &               \bf{0.65} &      0.91 \\
\approach{} $(\tau=.98)$     &            0.90 &           0.90 &               1.04 &      \bf{0.98} \\
\bottomrule
\end{NiceTabular}
    \label{tab:synthetic_node}
\end{table}
 
For synthetic datasets, unlike real-world datasets, we have the liberty of having known ground truth explanations (GTE).
We report \textit{accuracy} and \textit{precision} of explanations by comparing them against the GTE in Table~\ref{tab:synthetic_node}. 
Note that for soft masking approaches, like \gnnexp{}, hard masks need to be constructed by a discretizing step that is choosing top-k important attributions.
In our experiments, we strengthen the soft-masking baselines by setting the $k$ to the exact size of the ground truth.
In addition, we also report the sparsity and \fidelity{} of corresponding explanations as for the earlier cases.

We observe that while the gradient-based methods achieve the highest accuracy, \approach{} achieves the best precision, sparsity, and \fidelity{}. 
The decrease in accuracy is due to two reasons. 
First, the higher accuracy of gradient-based methods is due to our decision to discretize the soft-masking baselines by allowing them active knowledge of the GTE size. 
On the other hand, \approach{}, natively outputs hard masks agnostic to the GTE size. 
PGE performs worst in terms of precision and accuracy.
SubgraphX retrieves explanations similar to Grad, but has the drawback of very high runtime, see Appendix~\ref{sec:experiments}.
To fully answer RQ~\ref{rq:groundtrutheval}, we also compared the achieved faithfulness of the explanation methods. 
Table~\ref{tab:syn2_faith} shows the model's and explainers' accuracies at different epochs. 
Even at epoch 1 with random weights, we see that the baselines achieve high precision on this synthetic dataset. 

\approach{} achieves the first accuracy peak at 200 epochs, where the model still cannot differentiate the motif from the BA nodes.
A similar peak at epoch 200 is observed for the \gnnexp{} and GradInput. 
Moreover, the gradient-based methods achieve the best or close to the best performance for the untrained GCN. 
From these observations, we conclude that the underlying assumption -- the better the GNN, the better the explanation -- not necessarily need to hold. 
In this synthetic setting, the GNN only needs to differentiate the two communities, and all but one explanation method can explain the house motif. 

In conclusion, for RQ~\ref{rq:groundtrutheval}, \approach{} outperforms all baselines in terms of precision and faithfulness while gradient-based methods achieve the highest accuracy.
Note that the good performance of gradient-based approaches conflicts with our conclusions when experimenting with real-world data.
We believe that this is a possible threat to be aware of while evaluating explanations. 
Specifically, some explanation methods perform admirably in more straightforward and synthetic cases but are not robust and do not generalize well when used in real-world scenarios.
However, \approach{} is quite robust to different types of models, data, and evaluation regimes.

\begin{table}[ht!]
    \centering
    \small
    \setlength{\tabcolsep}{3.5pt}
    \caption{Experiments on faithfulness according to \cite{sanchez2020evaluating} measured with Kendall's tau $\tau_{\text{Kendall}}$ of the retrieved explanation precision and test accuracy. To simulate different model performances, we saved the GCN model during different epochs on the synthetic dataset. 
    For \approach{} $\tau=.85$, the ordering of the explanations' performances nearly perfectly align with the order of the models performance. 
    }
    \begin{NiceTabular}{lrrrrrr r}
\toprule
\textbf{Method} &  \textbf{1}    &  \textbf{200}  &  \textbf{400}  &  \textbf{600}  &  \textbf{1400} &  \textbf{2000} & $\boldsymbol{\tau_{\text{Kendall}}}$ \\
\midrule

\gnnexp{}   &  0.50 &  0.54 &  0.41 &  0.40 &  0.37 &  0.40 & $-0.73$ \\
PGM      & 0.83 &  0.47 & 0.68 &  0.71 &  0.76 &  0.75 & 0.20 \\
PGE &0.20  &0.19   &0.23 & 0.21 & 0.23  & 0.20  & 0.36 \\
Grad           &  \bf{0.94} &  0.80 &  0.62 &  0.73 &  0.84 &  0.87 & 0.07 \\
GradInput      &  0.88 &  0.89 &  0.78 &  0.79 &  0.87 &  0.89 & 0.07 \\
\approach{} $(\tau=.85)$           &   0.00 &  \bf{0.92} &  \bf{0.88} &  \bf{0.93} &  \bf{0.94} &  \bf{0.94} & \bf{0.73} \\
\approach{} $(\tau=.98)$     &   0.00 &  0.90 &  0.85 &  0.84 &  0.87 &  0.90 & 0.47 \\
\bottomrule
\end{NiceTabular}
    \label{tab:syn2_faith}
\end{table}
\begin{figure*}[htbp]
    \centering
    \includegraphics[width=0.9\textwidth]{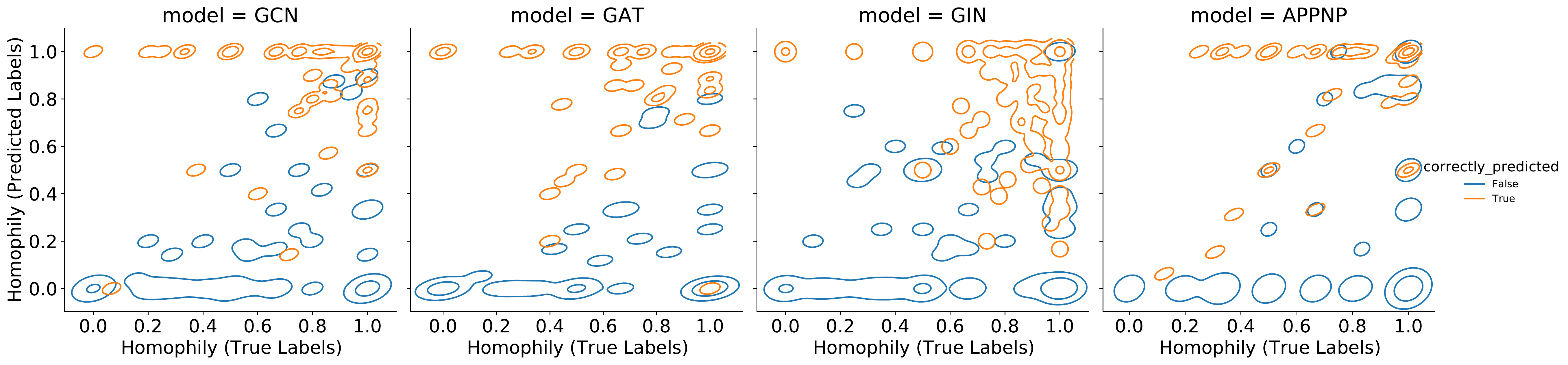}
    \caption{Dataset - PubMed. The joint distribution of the homophily with respect to the nodes selected in the \approach{}'s explanation ($\tau=.85$) with true and predicted labels. The orange contour lines correspond to the distributions for correctly predicted nodes, and the blue one corresponds to incorrectly predicted nodes.}
    \label{fig:homophily_predictions}
\end{figure*}

\section{Utility of Explanations}
\label{sec:homophily-explanations}

One of the motivations of post-hoc explanations is to derive insights into a model's inner workings. 
Specifically, we use explanations to analyze the behavior of different GNNs with respect to \textit{homophily}.
GNNs are known to exploit the homophily in the neighborhood to learn powerful function approximators. We use the retrieved explanations by \approach{} to verify the models' tendency to use homophily for node classification and identify the model's mistakes.

Formally, we define the \textit{homophily} of the node as the fraction of the number of its neighbors, which share the \textit{same label} as the node itself. 
In what follows, we use homophily to refer to the homophily of a node with respect to the selected nodes in its explanation. \textit{True homophily} is computed based on the true labels of the neighbor nodes. 
Similarly, \textit{predicted homophily} is computed based on the predicted labels of the neighboring nodes.

\subsection{Wrong Predictions Despite High Homophily} 

We start to investigate the joint density of true and predicted homophily of a given node. 
In Figure~\ref{fig:homophily_predictions}, we illustrate the effect of connectivity and neighbors' labels on the model's decision for a query node (for the PubMed dataset).
Several vertices corresponding to blue regions spread over the bottom of the plots have low predicted homophily. 
These nodes are incorrectly predicted, and their label differs from those predicted for the nodes in their explanation set. 
The surprising fact is that even though some of them have high true homophily close to 1, their predicted homophily is low. This also points to the usefulness of our found explanation in which we conclude that nodes influencing the current node do not share its label. 
So despite the consensus of GNNs reliance on homophily, they can still make mistakes for high homophily nodes, for example, when information from features is misaligned (or leads to a different decision) with that from the structure.

\subsection{Incorrect High Homophily Predictions} 

We also note that for GIN and APPNP, we have some nodes with true homophily and predicted homophily close to 1 but are incorrectly predicted. This implies that the node itself and the most influential nodes from its computational graph have been assigned the same label. We can conclude that the model based its decision on the right set of nodes but assigned the wrong class to the whole group. 

\subsection{Influence of Nodes with Low Homophily}

Nodes in the orange regions on the extreme left side of the plots exhibited low true homophily but high predicted homophily. The class labels for such nodes are correctly predicted. However, the corresponding nodes in the explanation were assigned the wrong labels (if they were assigned the same labels as that of the particular node in question, its predicted homophily would have been increased). The density of such regions in APPNP is lower than in GCN, implying that APPNP makes fewer mistakes in assigning labels to neighbors of low homophily nodes. For example, there are no nodes with true homophily $0$, which incorrectly influenced its neighbors. These nodes can be further studied with respect to their degree and features. 

\section{Conclusion}
We formulated the key properties a GNN explanation should follow: \textit{validity}, \textit{sparsity}, and \textit{stability}. While none of these measures alone suffice to evaluate a GNN explanation, we introduce a new metric called \textit{fidelity} that reflects these desiderata into a single measure. We provide theoretical foundations of fidelity from the area of \textit{rate-distortion theory}.
Furthermore, we proposed a simple combinatorial procedure \approach{}, which retrieves sparse \textit{binary masks} for the features and relevant nodes while trying to optimize for fidelity. Our experimental results on synthetic and real-world datasets show massive improvements not only for fidelity but also concerning evaluation measures employed by previous works.

\appendices

\ifCLASSOPTIONcompsoc
  \section*{Acknowledgments}
\else
  \section*{Acknowledgment}
\fi
The work is partially supported by the project "CampaNeo" (grant ID 01MD19007A) funded by the BMWi,  and the European Commission (EU H2020, ``smashHit", grant-ID 871477).

\ifCLASSOPTIONcaptionsoff
  \newpage
\fi

\bibliographystyle{abbrv}
\bibliography{references}

\pagebreak

\section{Additional details for the experiments}
\label{sec:experiments}

\mpara{Datasets.}
We use three well-known citation network datasets \textbf{Cora}, \textbf{CiteSeer} and \textbf{PubMed} from \citep{yang2016revisiting} where nodes represent documents and edges represent citation links. 
The class label is described by a similar word vector or an index of category. 
In addition, we use in Appendix~\ref{sec:app_amazon} the Amazon Computer dataset from~\citep{shchur2018pitfalls}, where nodes represent products and edges represent that two products are frequently bought together. The features are bag-of-words from product reviews, and the task is to predict the product category.
Statistics for these datasets can be found in Table~\ref{tab:dataset_stats}.
We used the datasets, including their training and test split from the PyTorch Geometric Library, which corresponds to the data published by \citep{yang2016revisiting} 
respectively~\citep{shchur2018pitfalls}. No default training and test split exist for the Amazon Computers dataset, so we generate one by randomly selecting 20 nodes of each class as the training set and another 1000 nodes as the test set.

Since the runtime of the explainers depends more on the size of the computational graphs than the size of the datasets, Table~\ref{tab:size_computational_graph} states the statistics of the computational graphs for the 300 randomly selected nodes and 2-layered GNNs. 
Since the citation datasets are sparse, these three datasets have smaller computational graphs. 
In contrast, the dense Amazon graph contains computational graphs of size more than 6000 nodes.

\begin{table}[htbp]
\centering
 
\caption{Datasets and statistics. The test accuracy is calculated on 1000 nodes.  
} 
\label{tab:dataset_stats}
\setlength{\tabcolsep}{3.5pt}
\begin{NiceTabular}{@{}lrrrrrrrr@{}}
\toprule
       &   &      &       &  & \multicolumn{4}{c}{\textbf{Test Accuracy}}  \\\cmidrule{6-9}
\textbf{Name }& \textbf{Class} & $d$ & $|V|$ & $|E|$& \textbf{GCN}    &     \textbf{GAT}     &     \textbf{GIN} & \textbf{APPNP}     \\
\midrule
Cora & 7 & 1433 & 2708  & 10556 &      0.794 &    0.791 &     0.679 & 0.799  \\
CiteSeer & 6 & 3703 & 3327  & 9104  &     0.675 &       0.673 &    0.480& 0.663 \\
PubMed & 3 & 500  & 19717 & 88648 &   0.782 &       0.765 &     0.590 &0.782  \\
Amazon & 10 & 767 & 13752 & 491722 & 0.787 & 0.797 & 0.471 & 0.810\\
\bottomrule
\end{NiceTabular}
\end{table}

\begin{table}[hb]
    \centering
    \caption{
    The number of nodes in the computational graph of the 300 randomly selected and explained nodes for 2-layer GNNs. 
    }
    \label{tab:size_computational_graph}
    \begin{NiceTabular}{lrrrr}
\toprule
{Dataset} &  Minimum &  Median &     Mean &  Maximum \\
\midrule
Cora     &     2 &    19.0 &    40.75 &   213 \\
CiteSeer &     1 &     7.0 &    14.28 &   137 \\
PubMed   &     3 &    32.0 &    56.34 &   527 \\
Amazon  &     1 &   708.0 &  1716.89 &  6428 \\
\bottomrule
\end{NiceTabular}
\end{table}

\mpara{Hyperparameters.}
For \approach{}, we retrieved explanations for the threshold $\tau=0.85$ and $\tau=0.98$ and with $K=10$. All \fidelity{} values were calculated based on $100$ samples. 
For all baselines, we use the default hyperparameters. 
The GNN on the synthetic dataset we trained on 80\% of the data and used Adam optimizer with a learning rate of $0.001$ and a weight decay of $0.005$ for $2000$ epochs.
We use the default training split of $10$ samples per class and applied Adam optimizer with a learning rate of $0.01$ and a weight decay of $0.005$ for $200$ epochs on the real datasets. 
We refer to the available implementation for further details of the models and the training of the GNNs.
We also include the saves of the trained model to increase the reproducibility. 
Our implementation is based on PyTorch Geometric 1.6~\citep{Fey/Lenssen/2019} and Python 3.7. 
All methods were executed on a server with 128 GB RAM and Nvidia GTX 1080Ti.

\mpara{Experiments on synthetic dataset.}
The synthetic dataset is generated by generating two communities consisting of house motifs attached to BA graphs. Each node has eight feature values drawn from $\mathcal{N}(0,1)$ and two features drawn from $\mathcal{N}(-1,.5)$ for nodes of the first community or $\mathcal{N}(1,.5)$ otherwise. In addition, to follow the published implementation of GNNExplainer, the feature values are normalized within each community, and within each community, $0.01\ \%$ of the edges are randomly perturbed. 

The eight labels are given by the following:
for each community, the nodes of the BA graph form a class, the 'basis' of the house forms a class, the 'upper' nodes form a class, and the rooftop is a class.
The used model is a three-layer GCN, which stacks each layer's latent representation and uses a linear layer to make the final prediction. 
The training set includes 80\% of the nodes. 

Since \gnnexp{} only returns soft edge mask, we sorted them and added both nodes from the highest-ranked edges until at least five nodes were selected. 
In this way, we retrieved hard node masks, which are necessary to compare with the ground truth. 

Table~\ref{tab:syn_runtime} states the runtime of the methods on the synthetic dataset.
The simple gradient-based methods, Grad and GradInput, are the fastest methods followed by PGE and \gnnexp{}. \approach{} $\tau=0.85$ is faster than \approach{} $\tau=0.98$, which needs to achieve a higher \fidelity{} threshold.
PGM takes around twice the time of \approach{} $\tau=0.98$. 
From the point of view of runtime, all the above methods are still usable in real-world post-hoc explanation scenarios. 
The only exception is SubgraphX, which already takes for this 3-layer GNN on the small synthetic dataset, on average more than one hour per node. 

\begin{table}
\small
    \centering
        \caption{Runtime in seconds on the synthetic dataset. We report the average runtime to retrieve an explanation for a single node and don't take any node-independent preprocessing into account. }
    \label{tab:syn_runtime}
    \begin{NiceTabular}{lr}
\toprule
Method &    Time \\
\midrule
\gnnexp{}   &   0.30 \\
PGM & 31.57\\
PGE &  0.05 \\ %
SubgraphX & 4895.03 \\
Grad           &   0.01 \\ %
GradInput      &   0.01\\ %
\approach{} ($\tau=.85$)          &   7.06 \\
\approach{} ($\tau=.98$)      &  15.26 \\
\bottomrule
\end{NiceTabular}
\end{table}

\mpara{Performance of model at different training epochs.}
Figure~\ref{fig:roar_accuracy} illustrates the performance during the epochs and Table~\ref{tab:roar_accuracy} states the details of the selected epochs. 
We selected the epochs such that the test accuracy increases monotonously. 
Furthermore, we reported the accuracy on the whole test set and only on the nodes belonging to the motif. 
For example, on epoch 200, the model has learned to differentiate the two communities but still cannot identify any node from the motif. 
Hence, the peak in precision for most explainers (see Table~\ref{tab:syn2_faith}) at epoch 200 is quite surprising. 

\begin{figure}
    \centering
    \includegraphics[width=0.7\linewidth]{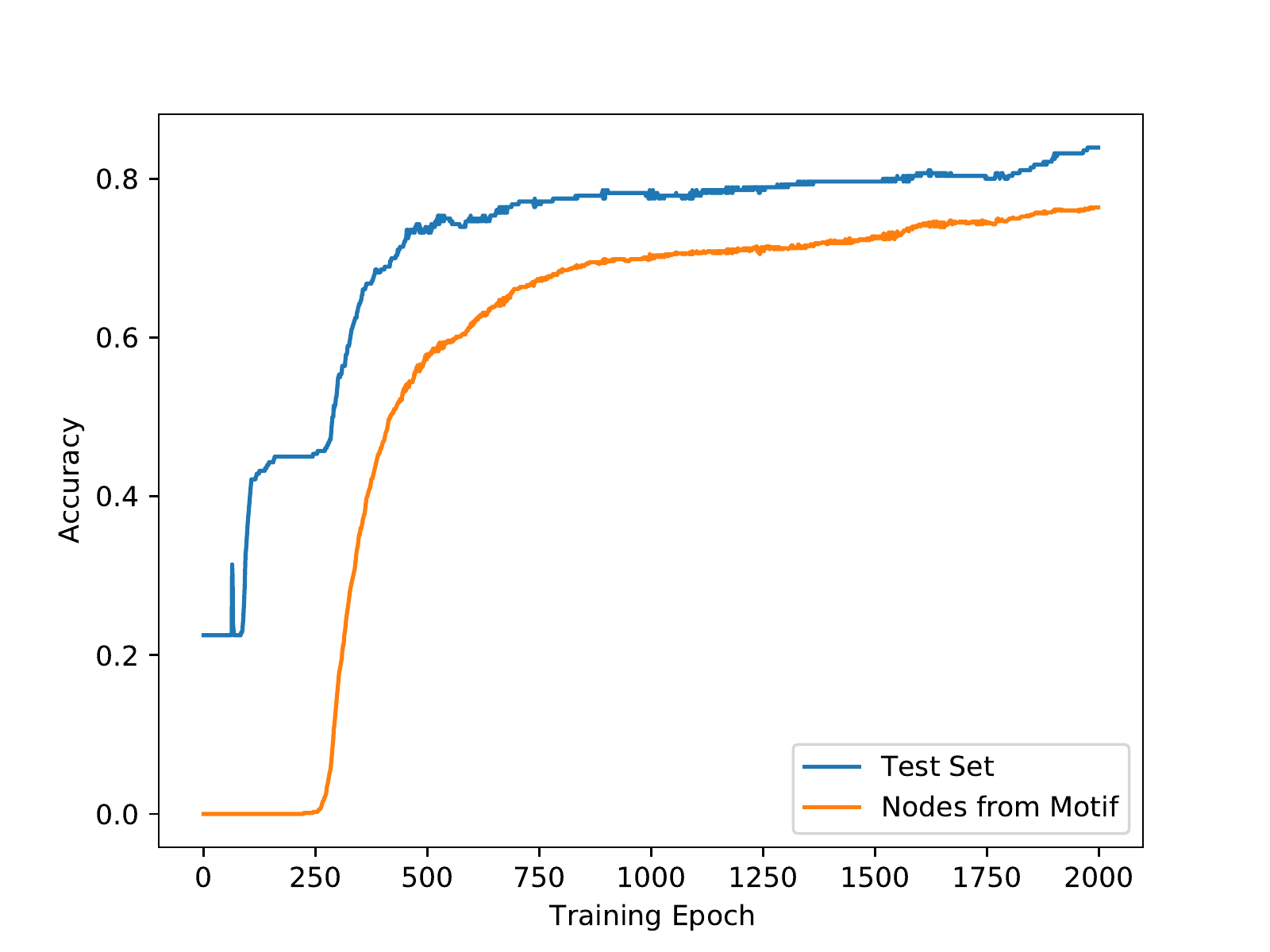}
    \caption{Accuracy of the GCN trained on synthetic dataset. We recorded the accuracy for the test set and the nodes in the motif, which are explained in the experiments.}
    \label{fig:roar_accuracy}
\end{figure}

\begin{table}
    \centering
        \caption{Accuracy of the selected epochs for the test set and the nodes in the motif.}
    \label{tab:roar_accuracy}
    \begin{tabular}{lrrr}
\toprule
  Epoch &  Test Accuracy &  Motif Accuracy \\
\midrule
      0 &           0.22 &            0.00 \\
    200 &           0.45 &            0.00 \\
   400 &           0.69 &            0.47 \\
    600 &           0.75 &            0.62 \\
   1400 &           0.80 &            0.72 \\
   2000 &           0.84 &            0.76 \\
\bottomrule
\end{tabular}
\end{table}

\section{Algorithmic Details}
\label{sec:algorithm_details}

\begin{figure*}
    \centering
    \includegraphics[width=0.7\textwidth]{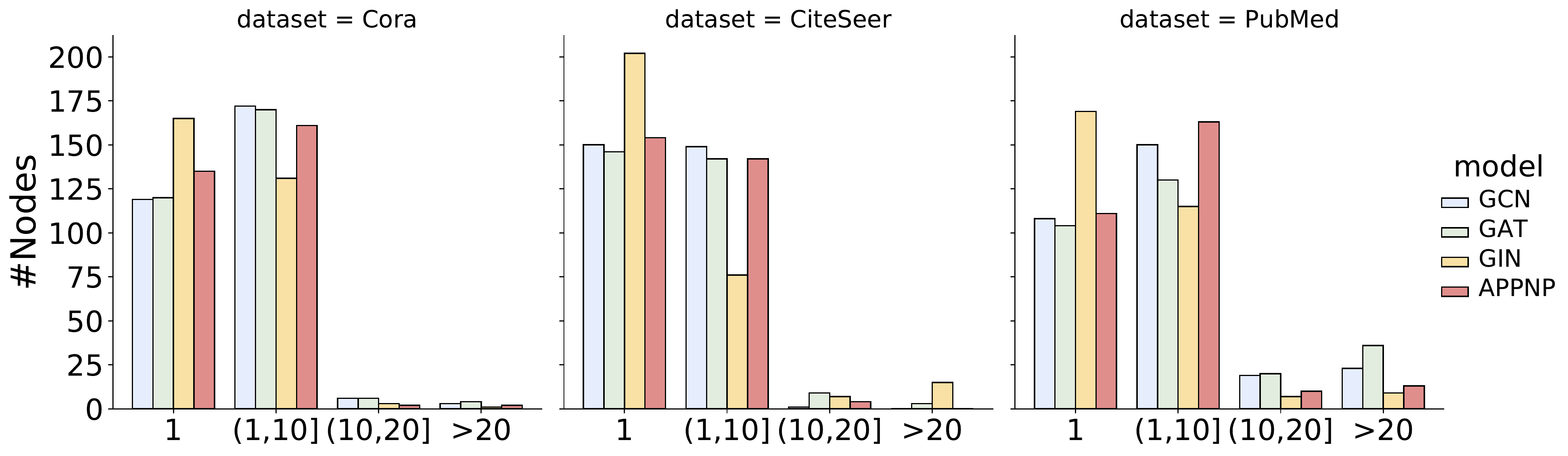}
    \caption{The number of explanations found with \approach{} at $\tau=0.85$.}
    \label{fig:multiplicity}
\end{figure*}

Multiple design choices can be considered in the design of \approach{}.
Specifically, the following design choices have an impact on the performance of \approach{}: (a) initialization of the first element and (b) iterative adding further elements.

\mpara{Initialization of first element.}
A single explanation $\{V_s, F_s\}$ consist of selected nodes $V_s$ and selected features $F_s$.
The challenge to select the first node and feature is the following:
Selecting only a node or only a feature yields a non-informative value, i.e., $\mathcal{F}(\{v\}, \emptyset) = c$ and $\mathcal{F}(\emptyset, \{f\}) = c$ for all $v\in V_n$ and $f\in F$ and some constant $c\in [0,1]$. 
The search for the optimal first pair would require $|V_p||F_p|$ evaluations of the fidelity, which is in most cases too expensive. 
Therefore, we propose to use a different strategy, which also contains information for the following iterations. 
Instead of evaluating, which pair of feature and node yields the highest increase, we assess the nodes and features in a maximal setting of the other. 
To be more precise, we assume that, if we search for the best node, all (possible) features $F_p$ were unmasked:
\begin{align}
\label{eq:node_argmax}
    \underset{v\in V_p}{\mathrm{argmax}} ~~\mathcal{F}(\{v\}, F_p)
\end{align}
Similarly for the features, we assume that all (possible) nodes are unmasked:
\begin{align}
\label{eq:feature_argmax}
    \underset{f\in F_p}{\mathrm{argmax}} ~~\mathcal{F}(V_p, \{f\})
\end{align}
Whichever of the nodes or features yields the highest value is the first element of our explanation. 
Consequently, the next selected element is different from the first element, e.g., if we first choose a node, the next element is always that feature, which yields the highest fidelity based on that single node. 
We perform this initialization again for each explanation since for each explanation, the maximal sets of possible elements $V_p$ and $F_p$ are different.

\begin{minipage}{0.46\textwidth}
\begin{algorithm}[H]
    \caption{\approach{}-Multiple$(n,\tau)$}
\label{alg:explain_multi}
\textbf{Input:} node $n$, threshold $\tau$ \\
\textbf{Output:} set of explanations with at least $\tau$-\fidelity{}
    \begin{algorithmic}[1]
\State $V_n\leftarrow$ set of vertices in $G(n)$
\State $F \leftarrow$ set of node features
\State \textbf{return} \approach{}-Recursive$(\tau,  V_n, F)$
\end{algorithmic}
\end{algorithm}

\begin{algorithm}[H]
    \caption{\approach{}-Recursive$(\tau, V_p, F_p)$}
\label{alg:real_recursive}
\textbf{Input:} threshold $\tau$, possible nodes $V_p$, possible feature $F_p$  \\
\textbf{Output:} set of explanations with at least $\tau$-\fidelity{}
    \begin{algorithmic}[1]
\State $\mathcal{S} = \emptyset$, $V_r=V_p$, $F_r=F_p$, $V_s=\emptyset$, $F_s=\emptyset$
\State $V_r=V_p$, $F_r=F_p$, $V_s=\emptyset$, $F_s=\emptyset$
\State {\small{$R_{V_p}\leftarrow $ list of $v\in V_p$ sorted by $\mathcal{F}(\{v\}, F_p)$ }} 
\State $R_{F_p}\leftarrow$ list of $f\in F_p$ sorted by $\mathcal{F}(V_p,\{f\})$ 
\State Add maximal element to $V_s$ or $F_s$ as in (\ref{eq:init})
\While{$\mathcal{F}(V_s,F_s) \ge \tau $}
\State $\Tilde{V}_s = V_s \cup \underset{v\in \operatorname{top}_K(V_r)}{\mathrm{argmax}} \mathcal{F}(\{v\}\cup V_s, F_s)$
\State $\Tilde{F}_s= F_s \cup \underset{f\in \operatorname{top}_K(F_r)}{\mathrm{argmax}} \! \mathcal{F}(V_s, \{f\}\cup F_s)$
\If{$\mathcal{F}(\Tilde{V}_s,F_s) \le \mathcal{F}(V_s,\Tilde{F}_s)$}
\State $F_r=F_r\setminus \{f\}$, $F_s=\Tilde{F}_s$
\Else
\State $V_r=V_r\setminus \{v\}$, $V_s= \Tilde{V_s}$
\EndIf
\EndWhile
\State $\mathcal{S} = \mathcal{S} \cup \{V_s, F_s\}$
\State $\mathcal{S} = \mathcal{S} \cup \text{\approach{}-Recursive}(\tau, K, V_p, F_r)$ \label{algline:recursive_call1}
\State $\mathcal{S} = \mathcal{S} \cup \text{\approach{}-Recursive}(\tau, K, V_r, F_p)$ \label{algline:recursive_call2}
\State \textbf{return} $\mathcal{S}$ 
\end{algorithmic}
\end{algorithm}
\end{minipage}

\mpara{Iterative search.}
The next part of our algorithm, which is the main contributor to the computational complexity, is the iterative search for additional nodes and features after the first element. 
A full search of all remaining nodes and features would require $|V_r|+|F_r|$ fidelity computations.
To significantly reduce this amount, we limited ourselves to a fixed number $K$ nodes and features, see Algorithm~\ref{alg:explain_multi}.
To systematically select the $K$ elements, we use the information retrieved in the initialization by Eq.~(\ref{eq:node_argmax}) and (\ref{eq:feature_argmax}).
We order the remaining nodes $V_r$ and $F_p$ by their values retrieved for Eq.~(\ref{eq:node_argmax}) and (\ref{eq:feature_argmax}) and only evaluate the top $K$.
In Algorithm~\ref{alg:explain_multi}, we have denoted these orderings by $R_{V_p}$ and $R_{F_p}$ and the retrieval of the top K remaining elements by $\operatorname{top}_K(V_r, R_{V_p})$ and $\operatorname{top}_K(F_r, R_{F_p})$. 
We also experimented with evaluating all remaining elements but observed no performance gain or inferior performance to the above heuristic. 
As a reason, we could identify that in some cases, the addition of a single element (feature or node) could not increase the achieved fidelity. 
Using the ordering retrieved from the "maximal setting", we enforce that those elements are still selected, which contain valuable information with a higher likelihood.
In addition, we experimented with refreshing the orderings $R_{V_p}$ and $R_{F_p}$ after some iterations but observed similar issues as in the unrestricted search.

\subsection{Multiple Explanations using \approach{}}
\label{sec:appendix_multiple_explanations}
We also evaluated and experimented a recursive variant of \approach{} that can retrieve multiple explanations with the desired fidelity. 

\mpara{\approach{} Variant for Multiple Explanations.}
In Algorithm~\ref{alg:real_recursive} we provide the generalization of \approach{} to extract multiple disjoint explanations of fidelity at least $\tau$.

We explicitly designed our algorithm in a way such that we can retrieve multiple explanations, see line \ref{algline:recursive_call1} and line \ref{algline:recursive_call2} of Algorithm~\ref{alg:real_recursive}.
We recursively call the Algorithm~\ref{alg:real_recursive} twice, once with a disjoint node-set, the call in line \ref{algline:recursive_call2} (only elements from the remaining set of nodes $V_r$ can be selected), and similarly in line \ref{algline:recursive_call1} with a disjoint feature set.
Hence, the resulting explanation selects disjoint elements from the feature matrix since either the rows or columns are different from before. 
As greedy and fast stop criteria, we used each further iteration, the maximal reachable fidelity of $\mathcal{F}(V_p, F_p)$.

\mpara{Analysing Multiplicity of Explanations}

As last part of our experiments, we want to study a recursive variant of \approach{}.
In algorithm~\ref{alg:real_recursive}, we start from all possible nodes in the computational graph and all possible features and select a small set of those elements as explanations.
The basic idea, is to recursively repeat \approach{}'s process to check the complements of the selected features or selected nodes for further high-fidelity explanations. 

Using above recursive variant, we find that
multiple (disjoint) explanations of fidelity at least $\tau$ are indeed possible and frequent. 
Figure~\ref{fig:multiplicity} shows the number of nodes having multiple explanations. 
We observe that, without exception, all GNNs yield multiple disjoint explanations with $\approx 50\%$ of the 300 nodes under study have $2$ to $10$ explanations.
Our algorithm's disjoint explanations can be understood as a disjoint piece of evidence that would lead the model to the same decision. 
We expect a much larger number of overlapping explanations if the restrictive condition on disjointness is relaxed. 
However, the objective here is to show that a decision can be reached in multiple ways. Each explanation is a practical realization of a possible combination of nodes and features that constitutes a decision.
We are the first to establish the multiplicity of explanations for GNN predictions.

\subsection{Proof of Theorem~\ref{thm:zorro}}
\label{sec:proof_thm_zorro}

\begin{proof} 
\mpara{Property 1.} Note that 
 \[
    \mathcal{F}(V_n, F_p) = 1 
    \]
i.e., the trivial explanation without any perturbation achieves the \fidelity{} value of $1 \ge \tau$. For any other set $S=\{V_s,F_s\}$ such that $V_s\subset V_n$ and $F_s\subset F_p$, line 7 of Algorithm~\ref{alg:recursive} ensures that $S$ is a possible explanation set only if 
\[
    \mathcal{F}(V_n, F_s) \ge \tau, 
\]
thereby completing the proof.

\mpara{Property 2.} As discussed above that the largest explanation which is the trivial explanation has size $|V_n|$ + $|F_p|$. So the while loop in Line 7 of the Algorithm~\ref{alg:recursive} in the worst-case results in $O(\max{|V_n|,F_p})$ iterations.

In each iteration, a node or feature is selected based on the resulting \fidelity{}. The \fidelity{} is computed in Algorithm~\ref{alg:fid} over a fixed number of samples. One forward pass using the perturbed data is required to compute \fidelity{} for a sample, which takes time $t$.  Combining all the above the run time complexity of \approach{} is then given by  $O(t\cdot \max(|V_n|+|F|))$.
    
\mpara{Property 3.} follows from property 1 of \approach{}, Theorem~\ref{thm:fidelity} and observing the fact that $\tau(1-\tau)$ decreases with respect to $\tau$ for $\tau\ge 0.5$.
\end{proof}
\mpara{Choice of Noisy Distribution $\mathcal{N}$.} Our choice of using the global distribution of features as the noisy distribution ensures that only plausible feature values are used.
Besides, our choice does not increase the bias towards specific values, which we would have by taking fixed values such as $0$ or averages.
One might argue that the irrelevant components can be set to $0$ rather than any specific noisy value. However, this might lead to several side effects: in the case of datasets for which a feature value of $0$ is not allowed or has some specific semantics or for models with some specific pooling strategy, for example, minpool.  More specifically, the idea of an irrelevant component is not that it is missing, but its value does not matter. Therefore to account for the irrelevancy of certain components given our explanation, we need to check for multiple noisy instantiations for the unselected components.  
\section{Further experiments}
\label{sec:further_experiments}
\subsection{Variations of Number of Samples}
We now study the effect of the number of samples (used to compute \fidelity{} in Algorithm~\ref{alg:fid}) on \approach's performance. We retrieved the explanations for the dataset Cora and the model GCN using $10, 50,$ and $150$ samples and the used default value of $100$. 
Table~\ref{tab:vary_samples} shows the results for node-sparsity, feature-sparsity, \fidelity{} and validity.
Overall we observe that the explanations retrieved using fewer samples are denser than those with 100 or 150 samples. 
In other words, more elements were added to the explanations, and more iterations within \approach{}'s algorithm were needed. 
Increasing the number of samples from $10$ to $50$ yields sparser explanations. 
A further gain in sparsification is seen for the default value of $100$ samples. 
In contrast, the $150$ samples show a mixed results with denser results for \approach{} $(\tau= 0.85)$ and sparser for \approach{} $(\tau= 0.98)$ as compared to the corresponding results when using $100$ samples. 

\begin{table}[htbp]
    \centering
        \caption{Results on Cora dataset with GCN model and varying number of samples $s$ used for \fidelity{} during retrieving the explanations. The stated \fidelity{} is calculated using 100 random samples.}
    \label{tab:vary_samples}
\begin{NiceTabular}{crrrrr}
\toprule
                      \multirow{2}{*}{Explainer} & \multirow{2}{*}{$s$} &  Node-&  Feature- &  \multirow{2}{*}{\fidelity{}} &  \multirow{2}{*}{Validity} \\
 &  &  Sparsity&  Sparsity &   &   \\                      
\midrule
  & 10  &               1.23 &                  4.39 &      0.80 &      0.99 \\
 Zorro                     & 50  &               1.28 &                  2.23 &      0.86 &      1.00 \\
   $(\tau = .85)$                   & 100 &               1.28 &                  1.91 &      0.87 &      1.00 \\
                      & 150 &               1.31 &                  1.96 &      0.87 &      1.00 \\ \midrule
  & 10  &               1.39 &                  5.02 &      0.88 &      0.99 \\
  Zorro                    & 50  &               1.51 &                  3.13 &      0.96 &      1.00 \\
     $(\tau = .98)$                 & 100 &               1.58 &                  2.69 &      0.97 &      1.00 \\
                      & 150 &               1.58 &                  2.56 &      0.97 &      1.00 \\
\bottomrule
\end{NiceTabular}
\end{table}

\begin{table}[htbp]
    \centering
        \caption{
    Varying the threshold $\tau$ for the dataset Cora and the model GCN. 
    }
    \label{tab:vary_tau}
    \begin{NiceTabular}{lrrrr}
\toprule
\multirow{2}{*}{$\tau$} &  Node-&  Feature- &  \multirow{2}{*}{Fidelity} &  \multirow{2}{*}{Validity} \\
   &  Sparsity&  Sparsity &   &   \\ 
\midrule
0.20 &               0.36 &                  0.63 &      0.27 &      0.73 \\
0.50 &               0.72 &                  1.15 &      0.57 &      0.95 \\
0.80 &               1.21 &                  1.78 &      0.83 &      1.00 \\
0.90 &               1.39 &                  2.13 &      0.91 &      1.00 \\
0.92 &               1.43 &                  2.23 &      0.93 &      1.00 \\
0.94 &               1.48 &                  2.36 &      0.94 &      1.00 \\
0.96 &               1.54 &                  2.55 &      0.96 &      1.00 \\
\bottomrule
\end{NiceTabular}
\end{table} %

\begin{table*}[p]
\small
    \centering
    \caption{ Results on real dataset with retrieved hard-masks. For those methods which already retrieve hard-masks (HM), the values are identical to those in Table~\ref{tab:real_main}.
    }
    \label{tab:additional_metrics_1}
    {
    \small
    \setlength{\tabcolsep}{3.5pt}
    \begin{NiceTabular}{llc ccc ccc ccc ccc}
\toprule
\multirow{2}{*}{\textbf{Metric}} & \multirow{2}{*}{\textbf{Method}} &
\multirow{2}{*}{\textbf{Transf.}} &\multicolumn{4}{c}{\textbf{Cora}} & \multicolumn{4}{c}{\textbf{CiteSeer}} & \multicolumn{4}{c}{\textbf{PubMed}} \\
 & &  &   GCN &   GAT &   GIN & APPNP &      GCN &   GAT &   GIN & APPNP &    GCN &   GAT &   GIN & APPNP \\

\midrule
\multirow[c]{11}{*}{\rotcell{Features-Sparsity}} & \multirow[c]{3}{*}{\gnnexp{}} & S-0.5 & 6.57 & 6.57 & 6.57 & 6.57 & 7.52 & 7.52 & 7.52 & 7.52 & 5.52 & 5.52 & 5.52 & 5.52 \\
 &  & S-0.7 & 6.06 & 6.06 & 6.06 & 6.06 & 7.01 & 7.01 & 7.01 & 7.01 & 5.01 & 5.01 & 5.01 & 5.01 \\
 &  & NT & 7.27 & 7.27 & 7.27 & 7.27 & 8.22 & 8.22 & 8.22 & 8.22 & 6.21 & 6.21 & 6.21 & 6.21 \\
 & \multirow[c]{3}{*}{Grad} & S-0.5 & 6.57 & 6.57 & 6.57 & 6.57 & 7.52 & 7.52 & 7.52 & 7.52 & 5.52 & 5.52 & 5.52 & 5.52 \\
 &  & S-0.7 & 6.06 & 6.06 & 6.06 & 6.06 & 7.01 & 7.01 & 7.01 & 7.01 & 5.01 & 5.01 & 5.01 & 5.01 \\
 &  & NT & 4.58 & 4.66 & 5.09 & 4.56 & 4.63 & 4.71 & 4.95 & 4.63 & 4.88 & 4.98 & 5.37 & 4.92 \\
 & \multirow[c]{3}{*}{GradInput} & S-0.5 & 6.57 & 6.57 & 6.57 & 6.57 & 7.52 & 7.52 & 7.52 & 7.52 & 5.52 & 5.52 & 5.52 & 5.52 \\
 &  & S-0.7 & 6.06 & 6.06 & 6.06 & 6.06 & 7.01 & 7.01 & 7.01 & 7.01 & 5.01 & 5.01 & 5.01 & 5.01 \\
 &  & NT & 4.59 & 4.67 & 5.10 & 4.56 & 4.63 & 4.71 & 4.95 & 4.64 & 4.88 & 4.98 & 5.40 & 4.92 \\
 & \approach{}  $(\tau = .85)$  & HM & 1.91 & 2.29 & 3.51 & 2.26 & 1.81 & 1.84 & 3.67 & 1.97 & 1.60 & 1.52 & 2.38 & 1.75 \\
 & \approach{}  $(\tau = .98)$  & HM & 2.69 & 3.07 & 4.34 & 3.18 & 2.58 & 2.60 & 4.68 & 2.78 & 2.55 & 2.58 & 3.21 & 2.86 \\
 
 \midrule
 \multirow[c]{13}{*}{\rotcell{Node-Sparsity}} & \multirow[c]{3}{*}{\gnnexp{}} & S-0.5 & 2.23 & 2.26 & 2.24 & 2.27 & 1.37 & 1.38 & 1.34 & 1.39 & 2.51 & 2.53 & 2.25 & 2.55 \\
 &  & S-0.7 & 1.88 & 1.94 & 1.94 & 1.96 & 1.25 & 1.25 & 1.24 & 1.30 & 2.15 & 2.19 & 1.95 & 2.18 \\
 &  & NT & 2.66 & 2.67 & 2.66 & 2.67 & 1.72 & 1.72 & 1.70 & 1.72 & 2.91 & 2.90 & 2.89 & 2.91 \\
  & PGM    & HM &  2.06 & 1.82 &  1.66 & 1.99 & 1.47 &  1.59 & 1.10 & 1.54 & 1.64 &  1.16 & 1.62 &  2.93 \\
 & PGE & HM & 1.86 & 1.86 & 1.78 & 1.94 &1.48 & 1.40& 1.36 & 1.41 & 1.91 &  1.81 & 1.85 & 1.92  \\ 
 & \multirow[c]{3}{*}{Grad} & S-0.5 & 2.33 & 2.33 & 2.33 & 2.33 & 1.30 & 1.30 & 1.30 & 1.30 & 2.80 & 2.80 & 2.80 & 2.80 \\
 &  & S-0.7 & 1.83 & 1.83 & 1.83 & 1.83 & 1.08 & 1.08 & 1.08 & 1.08 & 2.26 & 2.26 & 2.26 & 2.26 \\
 &  & NT & 2.81 & 2.71 & 2.91 & 2.73 & 1.79 & 1.76 & 1.80 & 1.76 & 3.34 & 3.29 & 3.41 & 3.25 \\
 & \multirow[c]{3}{*}{GradInput} & S-0.5 & 2.33 & 2.33 & 2.33 & 2.33 & 1.30 & 1.30 & 1.30 & 1.30 & 2.80 & 2.80 & 2.80 & 2.80 \\
 &  & S-0.7 & 1.83 & 1.83 & 1.83 & 1.83 & 1.08 & 1.08 & 1.08 & 1.08 & 2.26 & 2.26 & 2.26 & 2.26 \\
 &  & NT & 2.81 & 2.78 & 2.95 & 2.72 & 1.78 & 1.76 & 1.81 & 1.74 & 3.31 & 3.23 & 3.42 & 3.14 \\
 & \approach{}  $(\tau = .85)$  & HM & 1.28 & 1.30 & 1.90 & 1.16 & 1.05 & 0.92 & 1.36 & 0.83 & 1.07 & 0.87 & 1.77 & 0.79 \\
 & \approach{}  $(\tau = .98)$  & HM & 1.58 & 1.59 & 2.17 & 1.48 & 1.26 & 1.09 & 1.58 & 1.07 & 1.51 & 1.31 & 2.18 & 1.25 \\
 
 \midrule
 \multirow[c]{13}{*}{\rotcell{\fidelity{}}} & \multirow[c]{3}{*}{\gnnexp{}} & S-0.5 & 0.86 & 0.86 & 0.71 & 0.83 & 0.82 & 0.82 & 0.63 & 0.78 & 0.81 & 0.82 & 0.71 & 0.79 \\
 &  & S-0.7 & 0.72 & 0.72 & 0.57 & 0.67 & 0.71 & 0.69 & 0.55 & 0.64 & 0.68 & 0.72 & 0.66 & 0.71 \\
 &  & NT & 0.98 & 0.98 & 0.94 & 0.98 & 0.98 & 0.99 & 0.88 & 0.99 & 0.95 & 0.98 & 0.88 & 0.96 \\
  & PGM & HM & 0.84 &  0.77 &  0.60 &  0.89 & 0.92 & 0.93 &  0.73 &  0.95 &   0.78 &  0.69 &  0.74 &  0.96 \\ 
 & PGE & HM & 0.50 & 0.53 & 0.35 & 0.49 &0.64 &0.60 & 0.51 & 0.61 & 0.49 & 0.61  & 0.56 & 0.50  \\ 
 & \multirow[c]{3}{*}{Grad} & S-0.5 & 0.89 & 0.91 & 0.74 & 0.88 & 0.84 & 0.84 & 0.58 & 0.82 & 0.90 & 0.91 & 0.65 & 0.88 \\
 &  & S-0.7 & 0.80 & 0.84 & 0.60 & 0.82 & 0.67 & 0.65 & 0.42 & 0.65 & 0.86 & 0.86 & 0.56 & 0.84 \\
 &  & NT & 0.89 & 0.90 & 0.77 & 0.88 & 0.84 & 0.84 & 0.59 & 0.81 & 0.90 & 0.90 & 0.78 & 0.88 \\
 & \multirow[c]{3}{*}{GradInput} & S-0.5 & 0.87 & 0.90 & 0.75 & 0.88 & 0.81 & 0.81 & 0.60 & 0.80 & 0.88 & 0.90 & 0.70 & 0.87 \\
 &  & S-0.7 & 0.77 & 0.82 & 0.60 & 0.80 & 0.63 & 0.63 & 0.45 & 0.63 & 0.83 & 0.84 & 0.61 & 0.81 \\
 &  & NT & 0.88 & 0.89 & 0.79 & 0.88 & 0.82 & 0.82 & 0.59 & 0.80 & 0.90 & 0.90 & 0.79 & 0.87 \\
 & \approach{}  $(\tau = .85)$  & HM & 0.87 & 0.88 & 0.86 & 0.88 & 0.87 & 0.86 & 0.87 & 0.86 & 0.86 & 0.88 & 0.88 & 0.87 \\
 & \approach{}  $(\tau = .98)$  & HM & 0.97 & 0.97 & 0.96 & 0.97 & 0.97 & 0.97 & 0.97 & 0.96 & 0.96 & 0.97 & 0.97 & 0.96 \\
 
 \midrule
 \multirow[c]{13}{*}{\rotcell{Validity}} & \multirow[c]{3}{*}{\gnnexp{}} & S-0.5 & 0.89 & 0.94 & 0.79 & 0.90 & 0.88 & 0.88 & 0.67 & 0.86 & 0.84 & 0.87 & 0.63 & 0.83 \\
 &  & S-0.7 & 0.80 & 0.86 & 0.72 & 0.81 & 0.82 & 0.83 & 0.63 & 0.79 & 0.66 & 0.77 & 0.65 & 0.80 \\
 &  & NT & 0.98 & 0.98 & 0.94 & 0.98 & 0.98 & 0.99 & 0.88 & 0.99 & 0.95 & 0.98 & 0.87 & 0.96 \\
 & PGM & HM & 0.89 &  0.90 &  0.64 &  0.94 & 0.95 & 0.95 &  0.76 &  0.97 &   0.86 &  0.80 &  0.62 &  0.97 \\
 & PGE & HM & 0.51 & 0.54 & 0.34 &0.45  &0.62 & 0.59& 0.54  & 0.62 & 0.51 &  0.61 &0.57  & 0.48  \\ 
 & \multirow[c]{3}{*}{Grad} & S-0.5 & 0.96 & 0.98 & 0.89 & 0.96 & 0.95 & 0.97 & 0.77 & 0.95 & 0.98 & 0.99 & 0.80 & 0.98 \\
 &  & S-0.7 & 0.91 & 0.93 & 0.78 & 0.90 & 0.81 & 0.77 & 0.53 & 0.80 & 0.94 & 0.97 & 0.68 & 0.96 \\
 &  & NT & 0.97 & 0.99 & 0.96 & 0.98 & 0.96 & 0.97 & 0.86 & 0.96 & 0.99 & 0.99 & 0.95 & 0.99 \\
 & \multirow[c]{3}{*}{GradInput} & S-0.5 & 0.94 & 0.96 & 0.87 & 0.94 & 0.90 & 0.90 & 0.76 & 0.90 & 0.95 & 0.98 & 0.86 & 0.96 \\
 &  & S-0.7 & 0.89 & 0.93 & 0.77 & 0.88 & 0.74 & 0.72 & 0.57 & 0.75 & 0.91 & 0.94 & 0.67 & 0.92 \\
 &  & NT & 0.97 & 0.97 & 0.97 & 0.98 & 0.93 & 0.92 & 0.85 & 0.93 & 0.98 & 0.98 & 0.97 & 0.98 \\
 & \approach{}  $(\tau = .85)$  & HM & 1.00 & 1.00 & 0.83 & 1.00 & 1.00 & 1.00 & 0.77 & 1.00 & 0.90 & 1.00 & 0.84 & 1.00 \\
 & \approach{}  $(\tau = .98)$  & HM & 1.00 & 1.00 & 0.90 & 1.00 & 1.00 & 1.00 & 0.91 & 1.00 & 0.98 & 1.00 & 0.87 & 1.00 \\
 \bottomrule
\end{NiceTabular}
    }
\end{table*}

\settowidth\rotheadsize{Fidelity+prob}
\begin{table*}[p]
\small
    \centering
    \caption{ Results on real dataset with retrieved hard masks evaluated with four fidelity variants.
    }
    \label{tab:additional_metrics_2}
    {
    \small
    \setlength{\tabcolsep}{3.5pt}
    \begin{NiceTabular}{llc ccc ccc ccc ccc}
\toprule
\multirow{2}{*}{\textbf{Metric}} & \multirow{2}{*}{\textbf{Method}} &
\multirow{2}{*}{\textbf{Transf.}} &\multicolumn{4}{c}{\textbf{Cora}} & \multicolumn{4}{c}{\textbf{CiteSeer}} & \multicolumn{4}{c}{\textbf{PubMed}} \\
 & &  &   GCN &   GAT &   GIN & APPNP &      GCN &   GAT &   GIN & APPNP &    GCN &   GAT &   GIN & APPNP \\
 
 \midrule
 \multirow[c]{13}{*}{\rotcell{Fidelity+acc}} & \multirow[c]{3}{*}{\gnnexp{}} & S-0.5 & 0.45 & 0.34 & 0.41 & 0.44 & 0.40 & 0.32 & 0.67 & 0.37 & 0.47 & 0.27 & 0.38 & 0.23 \\
 &  & S-0.7 & 0.31 & 0.24 & 0.34 & 0.27 & 0.29 & 0.24 & 0.55 & 0.26 & 0.33 & 0.23 & 0.37 & 0.20 \\
 &  & NT & 0.78 & 0.78 & 0.82 & 0.83 & 0.82 & 0.84 & 0.89 & 0.81 & 0.64 & 0.51 & 0.45 & 0.63 \\
 & {PGM} & HM & 0.18 & 0.16 &0.25 &0.22 & 0.39 &0.39 & 0.40 &0.40 & 0.17 & 0.10 & 0.28 & 0.23 \\ 
  & {PGE} & HM & 0.06 &0.05 & 0.16& 0.09&0.09&0.10 &0.26 & 0.11 & 0.04 & 0.05 & 0.15&0.02\\ 
 & \multirow[c]{3}{*}{Grad} & S-0.5 & 0.32 & 0.42 & 0.28 & 0.60 & 0.40 & 0.35 & 0.41 & 0.39 & 0.45 & 0.30 & 0.33 & 0.36 \\
 &  & S-0.7 & 0.21 & 0.22 & 0.19 & 0.33 & 0.16 & 0.14 & 0.21 & 0.17 & 0.28 & 0.20 & 0.24 & 0.21 \\
 &  & NT & 0.72 & 0.71 & 0.64 & 0.75 & 0.68 & 0.59 & 0.67 & 0.65 & 0.63 & 0.49 & 0.48 & 0.58 \\
 & \multirow[c]{3}{*}{GradInput} & S-0.5 & 0.24 & 0.34 & 0.29 & 0.56 & 0.26 & 0.22 & 0.45 & 0.27 & 0.38 & 0.25 & 0.33 & 0.27 \\
 &  & S-0.7 & 0.16 & 0.19 & 0.20 & 0.25 & 0.09 & 0.10 & 0.22 & 0.12 & 0.23 & 0.17 & 0.23 & 0.18 \\
 &  & NT & 0.71 & 0.70 & 0.72 & 0.76 & 0.63 & 0.51 & 0.69 & 0.60 & 0.62 & 0.48 & 0.53 & 0.49 \\
 & \approach{} $\tau=0.85$ & HM & 0.31 & 0.33 & 0.40 & 0.33 & 0.46 & 0.42 & 0.56 & 0.48 & 0.29 & 0.26 & 0.31 & 0.30 \\
 & \approach{} $\tau=0.98$ & HM & 0.39 & 0.40 & 0.50 & 0.46 & 0.52 & 0.53 & 0.68 & 0.57 & 0.37 & 0.35 & 0.42 & 0.42 \\
 
 \midrule 
 \multirow[c]{13}{*}{\rotcell{Fidelity-acc}} & \multirow[c]{3}{*}{\gnnexp{}} & S-0.5 & 0.11 & 0.06 & 0.21 & 0.10 & 0.12 & 0.12 & 0.33 & 0.14 & 0.16 & 0.13 & 0.37 & 0.17 \\
 &  & S-0.7 & 0.20 & 0.14 & 0.28 & 0.19 & 0.18 & 0.17 & 0.37 & 0.21 & 0.34 & 0.23 & 0.35 & 0.20 \\
 &  & NT & 0.02 & 0.02 & 0.06 & 0.02 & 0.02 & 0.01 & 0.12 & 0.01 & 0.05 & 0.02 & 0.13 & 0.04 \\
 & {PGM} & HM & 0.14 & 0.17 &0.55 &0.15 & 0.09 & 0.09 & 0.29 &0.05 & 0.24 & 0.24 & 0.51 & 0.10 \\ 
 & {PGE} & HM & 0.56 &0.51& 0.71 & 0.56&0.34 & 0.37 &0.63 & 0.37 & 0.52 & 0.41 & 0.71&0.44\\ 

 & \multirow[c]{3}{*}{Grad} & S-0.5 & 0.04 & 0.02 & 0.11 & 0.04 & 0.05 & 0.03 & 0.23 & 0.05 & 0.02 & 0.01 & 0.20 & 0.02 \\
 &  & S-0.7 & 0.09 & 0.07 & 0.22 & 0.10 & 0.19 & 0.23 & 0.47 & 0.20 & 0.06 & 0.03 & 0.32 & 0.04 \\
 &  & NT & 0.03 & 0.01 & 0.04 & 0.02 & 0.04 & 0.03 & 0.14 & 0.04 & 0.01 & 0.01 & 0.05 & 0.01 \\
 & \multirow[c]{3}{*}{GradInput} & S-0.5 & 0.06 & 0.04 & 0.13 & 0.06 & 0.10 & 0.10 & 0.24 & 0.10 & 0.05 & 0.02 & 0.14 & 0.04 \\
 &  & S-0.7 & 0.11 & 0.07 & 0.23 & 0.12 & 0.26 & 0.28 & 0.43 & 0.25 & 0.09 & 0.06 & 0.33 & 0.08 \\
 &  & NT & 0.03 & 0.03 & 0.03 & 0.02 & 0.07 & 0.08 & 0.15 & 0.07 & 0.02 & 0.02 & 0.03 & 0.02 \\
 & \approach{} $\tau=0.85$ & HM & 0.00 & 0.00 & 0.17 & 0.00 & 0.00 & 0.00 & 0.23 & 0.00 & 0.10 & 0.00 & 0.16 & 0.00 \\
 & \approach{} $\tau=0.98$ & HM & 0.00 & 0.00 & 0.10 & 0.00 & 0.00 & 0.00 & 0.09 & 0.00 & 0.02 & 0.00 & 0.13 & 0.00 \\
 
 \midrule
 \multirow[c]{13}{*}{\rotcell{Fidelity+prob}} & \multirow[c]{3}{*}{\gnnexp{}} & S-0.5 & 0.64 & 0.63 & 0.45 & 0.60 & 0.57 & 0.53 & 0.58 & 0.50 & 0.37 & 0.39 & 0.27 & 0.35 \\
 &  & S-0.7 & 0.56 & 0.53 & 0.35 & 0.49 & 0.46 & 0.42 & 0.48 & 0.39 & 0.32 & 0.34 & 0.26 & 0.30 \\
 &  & NT & 0.72 & 0.74 & 0.79 & 0.72 & 0.68 & 0.67 & 0.76 & 0.67 & 0.44 & 0.46 & 0.35 & 0.46 \\
  & {PGM} & HM & 0.26 & 0.21 &0.21 &0.28 & 0.40 &0.39 & 0.33 &0.41 & 0.14 & 0.10 & 0.25 & 0.20 \\
   & {PGE} & HM & 0.07 &0.06& 0.11& 0.08&0.14 & 0.13 &0.20 & 0.17 & 0.02 & 0.04 & 0.11&0.02\\ 
 & \multirow[c]{3}{*}{Grad} & S-0.5 & 0.53 & 0.54 & 0.25 & 0.56 & 0.44 & 0.44 & 0.34 & 0.46 & 0.35 & 0.36 & 0.23 & 0.36 \\
 &  & S-0.7 & 0.35 & 0.36 & 0.14 & 0.40 & 0.22 & 0.23 & 0.16 & 0.26 & 0.26 & 0.28 & 0.16 & 0.28 \\
 &  & NT & 0.68 & 0.69 & 0.60 & 0.67 & 0.60 & 0.57 & 0.59 & 0.58 & 0.43 & 0.43 & 0.37 & 0.42 \\
 & \multirow[c]{3}{*}{GradInput} & S-0.5 & 0.50 & 0.52 & 0.26 & 0.55 & 0.39 & 0.40 & 0.39 & 0.43 & 0.33 & 0.35 & 0.24 & 0.34 \\
 &  & S-0.7 & 0.31 & 0.34 & 0.16 & 0.38 & 0.18 & 0.21 & 0.17 & 0.23 & 0.24 & 0.26 & 0.14 & 0.26 \\
 &  & NT & 0.68 & 0.70 & 0.69 & 0.67 & 0.58 & 0.55 & 0.61 & 0.56 & 0.43 & 0.43 & 0.41 & 0.41 \\
 & \approach{}  $(\tau = .85)$  & HM & 0.26 & 0.25 & 0.35 & 0.29 & 0.34 & 0.30 & 0.48 & 0.34 & 0.16 & 0.16 & 0.24 & 0.18 \\
 & \approach{}  $(\tau = .98)$  & HM & 0.35 & 0.34 & 0.44 & 0.38 & 0.42 & 0.39 & 0.59 & 0.43 & 0.22 & 0.24 & 0.36 & 0.26 \\
 
 \midrule
 \multirow[c]{13}{*}{\rotcell{Fidelity-prob}} & \multirow[c]{3}{*}{\gnnexp{}} & S-0.5 & 0.28 & 0.25 & 0.20 & 0.25 & 0.26 & 0.28 & 0.26 & 0.25 & 0.22 & 0.23 & 0.26 & 0.21 \\
 &  & S-0.7 & 0.42 & 0.42 & 0.27 & 0.40 & 0.43 & 0.45 & 0.31 & 0.40 & 0.31 & 0.31 & 0.24 & 0.27 \\
 &  & NT & 0.10 & 0.07 & 0.05 & 0.08 & 0.09 & 0.09 & 0.07 & 0.07 & 0.09 & 0.09 & 0.06 & 0.07 \\
 & {PGM} & HM & 0.25 & 0.27 & 0.51 &0.20& 0.14 &0.12 & 0.24 &0.12 &0.22 & 0.28 & 0.48 & 0.13 \\
 & {PGE} & HM & 0.57 &0.53&0.65 & 0.54&0.42 & 0.40 &0.52 & 0.42 & 0.37 & 0.37 & 0.65&0.40\\ 

 & \multirow[c]{3}{*}{Grad} & S-0.5 & 0.04 & 0.02 & 0.09 & 0.04 & 0.06 & 0.05 & 0.18 & 0.06 & 0.03 & 0.02 & 0.11 & 0.02 \\
 &  & S-0.7 & 0.12 & 0.08 & 0.19 & 0.11 & 0.24 & 0.23 & 0.38 & 0.23 & 0.08 & 0.06 & 0.22 & 0.06 \\
 &  & NT & 0.01 & 0.01 & 0.03 & 0.01 & 0.02 & 0.03 & 0.08 & 0.02 & 0.00 & 0.00 & 0.03 & 0.00 \\
 & \multirow[c]{3}{*}{GradInput} & S-0.5 & 0.06 & 0.04 & 0.10 & 0.05 & 0.09 & 0.09 & 0.18 & 0.08 & 0.04 & 0.03 & 0.07 & 0.03 \\
 &  & S-0.7 & 0.15 & 0.10 & 0.19 & 0.13 & 0.27 & 0.25 & 0.33 & 0.25 & 0.10 & 0.08 & 0.22 & 0.07 \\
 &  & NT & 0.02 & 0.01 & 0.01 & 0.02 & 0.04 & 0.05 & 0.08 & 0.04 & 0.01 & 0.01 & 0.01 & 0.01 \\
 & \approach{}  $(\tau = .85)$  & HM & 0.18 & 0.12 & 0.12 & 0.14 & 0.11 & 0.11 & 0.18 & 0.06 & 0.21 & 0.19 & 0.05 & 0.15 \\
 & \approach{}  $(\tau = .98)$  & HM & 0.08 & 0.02 & 0.04 & 0.03 & 0.01 & -0.00 & 0.01 & -0.04 & 0.13 & 0.10 & 0.03 & 0.05 \\
 
 \bottomrule
\end{NiceTabular}
    }
\end{table*} 
\subsection{Further Variation of the \fidelity{} Threshold}
In the main experiments, we have evaluated the explanations with two thresholds of $\tau$, that is $0.85$ and $0.98$. 
Our implementation fo \approach{} returns not only the node and the feature mask but also the list of selected elements (node/features) with their corresponding \fidelity{}. 
Hence, all explanations with $\tau' < \tau$ can be computed with negligible computational cost. 
Table~\ref{tab:vary_tau} states the results for different \fidelity{} thresholds on the Cora dataset with the GCN model.
As expected, more selected elements are needed to reach higher \fidelity{} values. 
In other words, the node-masks and the feature-masks become denser.

\subsection{Further Evaluation Metrics}
\label{sec:app_further_metrics}

To confirm our findings of Section~\ref{sec:realeval}, we also evaluated with four metrics of \citep{yuan2020explainability}. 
First, we recap the definitions of these metrics and discuss the underlying assumptions. 
Since these metrics assume hard-mask explanations, we then state different transformations of soft-mask to hard-mask explanations. 
Afterward, we discuss the results and compare them with our fidelity, sparsity, and validity. 
For explanations in general many different metrics are proposed, see for example \cite{singh2021valid,singh2021extracting}.

\subsubsection{Definition and Discussion of Additional Metrics}
The four fidelity variants proposed in \citep{yuan2020explainability} are defined as follows:

\begin{definition}[Fidelity+acc~\citep{yuan2020explainability}]
The Fidelity+acc of explanation $\mathcal{S}$ is $0$ if $\Phi(X\backslash \mathcal{S})=\Phi(X)$ and $1$ otherwise. 
\end{definition}
\begin{definition}[Fidelity-acc~\citep{yuan2020explainability}]
The Fidelity-acc of explanation $\mathcal{S}$ is $0$ if $\Phi(\mathcal{S})=\Phi(X)$ and $1$ otherwise. 
\end{definition}
Fidelity-acc of an explanation is exactly $1-$ validity score of that explanation.

\begin{definition}[Fidelity+prob~\citep{yuan2020explainability}]
The Fidelity+prob of explanation $\mathcal{S}$ is given by 
\begin{align*}
    h(X) - h(X\backslash \mathcal{S}),
\end{align*}
where $h(\cdot)$ is the predicted probability of $\Phi$ for the predicted class. 
\end{definition}
\begin{definition}[Fidelity-prob~\citep{yuan2020explainability}]
The Fidelity-prob of explanation $\mathcal{S}$ is given by 
\begin{align*}
    h(X) - h(\mathcal{S}).
\end{align*}
\end{definition}

According to~\citep{yuan2020explainability}, high values of Fidelity+prob and Fidelity+acc and low values of Fidelity-prob and Fidelity-acc are indicative of good explanations. We argue that these scores cannot alone completely evaluate the goodness of an explanation and can often be misleading. First, note that using the complete input as a trivial explanation already achieves maximum scores for Fidelity+acc and Fidelity+prob measures. Second, Fidelity-acc and Fidelity-prob assume that an explanation is complete, i.e., any subset of input that is not included in the explanation is not valid. Such an explanation might not always exist. More concretely, we show in Appendix~\ref{sec:appendix_multiple_explanations} that one can retrieve multiple explanations which would obtain high scores with Fidelity+acc and Fidelity+prob.

\subsubsection{Deriving Hard Masks}
The above four metrics require hard-masks~\citep{yuan2020explainability}. 
Hence, for all methods that yield soft masks, an additional transformation into hard masks is needed. We employ two approaches: 
\begin{itemize}
    \item For each soft mask (feature mask, edge mask, node mask), we select the $x\%$ entries with the highest soft-mask value. We follow~\citep{yuan2020explainability} and use $30\%$ and $50\%$. We denote them with S-0.5 and S-0.7. 
    \item \cite{pope2019explainability} proposed to normalize the soft masks to values between $0$ and $1$ and keep all elements greater than $0.01$. We denote this transformation with $NT$.
\end{itemize}
One challenge for any transformation of soft masks to hard-masks is the introduced trade-off between sparsity and explanation performance. 
Especially with varying sizes of elements to choose from, e.g., a varying number of nodes in the computational graph, the selection of a transformation is a non-trivial task and can introduce bias in the evaluation. Besides, the transformation approaches are applied independently for different masks. Therefore, they cannot address the trade-off of adding elements in different masks.

\subsubsection{Results of Additional Metrics}
Table~\ref{tab:additional_metrics_1} contains the results with respect to our set of metrics for explanations. 
Table~\ref{tab:additional_metrics_2} contains the results with respect to the four fidelity variants. 

As our sparsity measures in Table~\ref{tab:additional_metrics_1} show, the three transformations of soft-masks to hard masks result in all cases in dense masks. 
The normalizing and threshold approach (NT) has a greater variance in the sparsity and depends on the distribution of the soft masks. 
For \gnnexp{}, the retrieved soft-masks are in the shape of a normal distribution. 
Hence, NT selects most of the elements and results in very dense masks for \gnnexp{}.
For Grad and GradInput, NT retrieves feature masks that are sparser than S-0.7. 
However, for the respective node masks, NT retrieves denser masks as compared to S-0.5. 
\approach{} retrieves by far the sparsest hard-mask explanations. 
The next sparsest approaches are PGM and PGE. 
The different sparsity levels also affect the comparability of results with respect to different evaluation metrics. 

With respect to \fidelity{}, \approach{} outperforms all baselines, except for \gnnexp{} combined with NT, which selects almost all the  input elements. 
A similar pattern is observed for validity, where \approach{} is only inferior to \gnnexp{} combined with NT in a single case. 

For the additional metrics in Table~\ref{tab:additional_metrics_2}, we first observe that \approach{} performs well with respect to fidelity-acc and fidelity-prob.
A separate analysis is not needed since fidelity-acc is given by 1 - validity. 
For fidelity-prob, \approach{} achieves among the lowest scores, which stands for a high quality of the explanations. 
Taking \approach{}'s sparser masks into account, we clearly outperform all other baselines.
For fidelity+acc, we observe that \approach{} performs similar to the baselines with the S-0.7 transformation. 
For fidelity+prob, we observe a more diversity in performance over different datasets and GNN models. 
In general, \approach{} $\tau=0.98$ achieves higher scores in fidelity+acc and fidelity+prob than \approach{} $\tau=0.85$ in all cases, which indicates that we add discriminative elements to achieve a higher \fidelity{}. 
Moreover, the other hard-masking approaches, PGM and PGE, perform inferior to \approach{}. 

Overall, these experiments support our finding of Table~\ref{tab:real_main} that \approach{} retrieves sparse explanations of high quality. 

\subsection{Experiment on Amazon Computers}
\label{sec:app_amazon}
To evaluate the scalability of our \approach{} approach, we evaluated with the Amazon Computers dataset, which is denser than the citation graphs. 
Hence, as Table~\ref{tab:size_computational_graph}, the size of the computational graph for Amazon Computers is one magnitude larger than those of the citation graphs. 
Note that increasing the size of the dataset, in the sense of a higher number of overall nodes, while keeping the average degree similar has no effect on \approach{}'s runtime for retrieving the explanation of a node. 
Table~\ref{tab:results_amazon} contains the results with respect to our four metrics validity, \fidelity{}, feature-sparsity and node-sparsity.

For all methods, except PGM and PGE, we observe an increase in the computed node-sparsity, which indicates that more nodes are needed to explain the prediction of the dense Amazon dataset. 
In contrast, PGM and PGE still retrieve only a few nodes. 
However, PGM and PGE yield in all but one case explanations with low \fidelity{} and validity score. 
Hence, these methods always derive sparse explanations even if the actual discriminative subset is large. 
Grad and GradInput also retrieve explanations with low \fidelity{} and validity scores. 
\gnnexp{} shows a more diverse behavior. 
For GCN, GAT, and APPNP, the validity score is high. The \fidelity{} scores are, however, low, which indicates that these explanations are not stable. 

\approach{} retrieves sparser explanations as compared to \gnnexp{}, Grad, and GradInput. 
The \fidelity{} scores for \approach{} is, as expected, very high in all cases. 
\approach{}'s validity score varies between the GNN models. 
For GCN and APPNP, we observe high validity scores. 
For GIN, the score of $0.45$, respectively, $0.48$ are the highest of the observed values. 
In general, this supports the pattern of Table~\ref{tab:real_main} that the validity scores of GIN are the lowest of all four GNNs. 
For GAT, the validity of \approach{} is only inferior to \gnnexp{}. 

Overall, we see our observations of Section~\ref{sec:realeval} confirmed on this dense Amazon dataset. 

\begin{table*}[htbp]
    \centering
        \caption{Results on Amazon Computers Dataset}
    \label{tab:results_amazon}
    \begin{NiceTabular}{llrrrr}
\toprule
    Model & Method & Validity & \fidelity{} & Feature-Sparsity & Node-Sparsity \\
\midrule
\multirow{5}{*}{GCN} & \gnnexp{} &     0.77 &     0.43 &                 6.64 &              6.14 \\
    & PGM &     0.19 &    0.24 &                 --- &              2.18\\
    & PGE &     0.16 &    0.15 &                 --- &              1.96\\
    & Grad &     0.32 &     0.12 &                 6.16 &              5.41 \\
    & GradInput &     0.31 &     0.12 &                 6.18 &              5.47 \\

    & \approach{} $(\tau = 0.85)$ &     0.86 &     0.88 &                 5.62 &              4.75 \\
    & \approach{} $(\tau = 0.98)$ &     0.93 &     0.98 &                 5.72 &              4.78 \\\midrule
\multirow{5}{*}{GAT} & \gnnexp{} &     0.75 &     0.58 &                 6.64 &              6.14 \\
    & PGM &     0.25 &  0.24 &                 --- &              2.30\\
    & PGE &     0.18 &    0.19 &                 --- &              1.85\\
    & Grad &     0.15 &     0.10 &                 5.99 &              5.08 \\
    & GradInput &     0.15 &     0.10 &                 6.08 &              5.38 \\
    & \approach{} $(\tau = 0.85)$ &     0.34 &     0.88 &                 5.02 &              3.99 \\
    & \approach{} $(\tau = 0.98)$ &     0.47 &     0.98 &                 5.37 &              4.18 \\ \midrule
\multirow{5}{*}{GIN} & \gnnexp{} &     0.26 &     0.26 &                 6.64 &              6.10 \\
    & PGM &     0.17 &    0.27 &                --- &              1.87\\
    & PGE &    0.10  &    0.54 &        ---&        1.69 \\
    & Grad &     0.15 &     0.26 &                 6.20 &              5.91 \\
    & GradInput &     0.17 &     0.26 &                 6.20 &              6.10 \\
    & \approach{} $(\tau = 0.85)$ &     0.45 &     0.90 &                 6.15 &              5.84 \\
    & \approach{} $(\tau = 0.98)$ &     0.48 &    0.99 &                 6.16 &              5.85 \\\midrule
\multirow{5}{*}{APPNP} & \gnnexp{} &     0.95 &     0.30 &                 6.64 &              6.15 \\
    & PGM &     0.82 &    0.38 &                 --- &              2.00\\
    & PGE &    0.37  &    0.15 &                 --- &              2.15\\
    & Grad &     0.31 &     0.19 &                 6.13 &              4.97 \\
    & GradInput &     0.30 &     0.18 &                 6.15 &              5.05 \\
    & \approach{} $(\tau = 0.85)$ &     0.82 &     0.92 &                 6.35 &              5.64 \\
    & \approach{} $(\tau = 0.98)$ &     0.82 &     0.99 &                 6.35 &              5.63 \\     
\bottomrule
\end{NiceTabular}
\end{table*} 
\section{Visualization of Explanation Examples for the Synthetic Dataset}
For a qualitative evaluation of the evaluated explainers, we also include in Figure~\ref{fig:syn2_legend} and Figure~\ref{fig:syn2_examples} visualization of the ground truth explanation (GTE) and explanations for a correctly and falsely predicted node. 
Figure~\ref{fig:syn2_legend} a) shows the general structure with the "two houses" on the left and right. 
For each node in a house, the GTE is all the nodes of its house. 
As we see in Figure~\ref{fig:syn2_legend}, even for the untrained model at epoch 1, the explainer Grad retrieves nearly perfect explanations, which is in line with our quantitative measures of Table~\ref{tab:syn2_faith}.
From Figure~\ref{fig:syn2_examples}, we observe that the explainers disagree on the reason for the prediction. 
\gnnexp{} and SubgraphX retrieve for both the correctly and the falsely predicted nodes a similar mask with some nodes in the GTE and a few nodes from the BA community. 
PGM retrieves four nodes of the GTE and one different direct neighbor for both cases. 
PGE retrieves for the correct prediction the same explanation as \approach{} $\tau=0.85$. 
However, for the false prediction, PGE retrieves three nodes of the GTE and two further nodes of the BA community. 
Grad and GradInput show a similar pattern to \gnnexp{} and SubgraphX, but contain more nodes of the GTE and fewer nodes of the community in the correctly classified node than in the falsely labeled node. 
\approach{} retrieves for both nodes less than five nodes, which is the size of the GTE. 
\approach{} $\tau=0.98$ retrieves more nodes than \approach{} $\tau=0.85$. 
All retrieved nodes are part of the GTE for the correctly classified node.
All retrieved nodes are not part of the GTE but nodes of a house from the other community for the falsely labeled node. 

\begin{figure*}[htbp]
    \centering
    \begin{tabular}{ccc}
           \includegraphics[width=150pt]{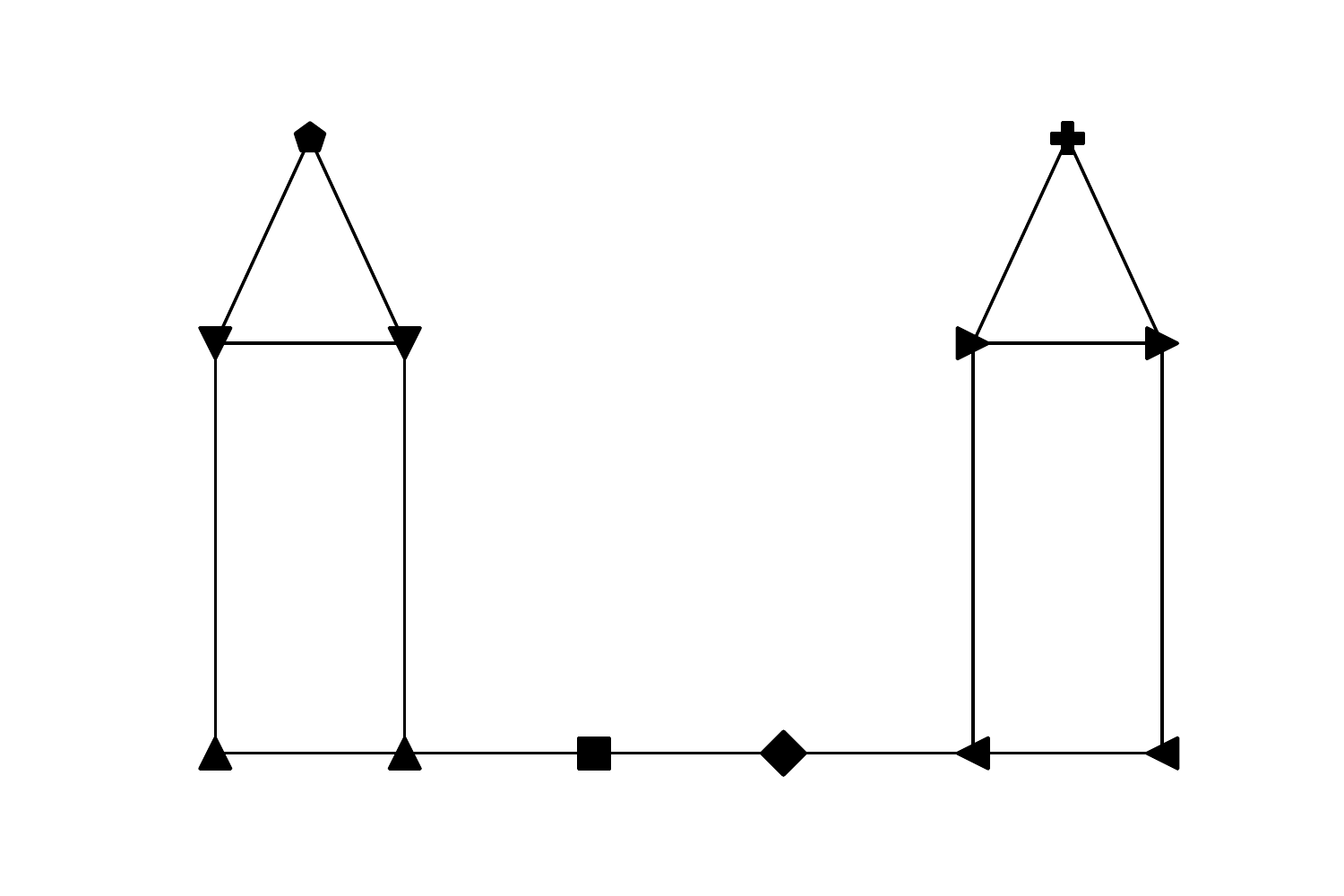}  & \includegraphics[width=100pt]{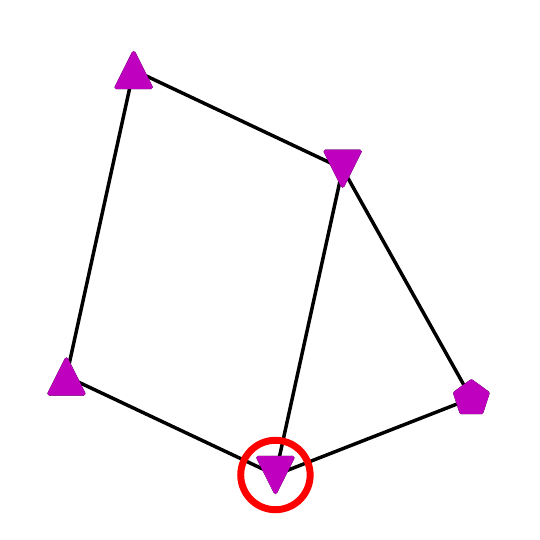}
           & \includegraphics[width=100pt]{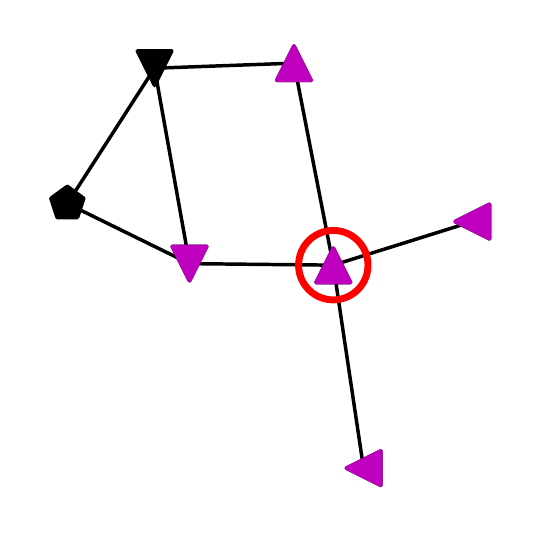}\\ 
          a) Legend & b) Node 300 & c) Node 302 \\
        \end{tabular}
    \caption{
    a) Representation of the different classes in the synthetic dataset with different symbols. Figure b) and c) visualize example explanations retrieved with Grad at epoch 1. 
    The colored nodes are contained in the explanation, and for simplicity, only nodes of the ground truth explanation and the explanation are shown. The circled node highlights the explained node. 
    }
    \label{fig:syn2_legend}
\end{figure*}

\begin{figure*}[p]
    \centering
    \begin{tabular}{|c|c||c|c|}
    \hline
     Node 300 & Node 300  & Node 302  & Node 302 \\
      (Correct Prediction) &  (Correct Prediction) &  (False Prediction) &  (False Prediction) \\
     \hline
     \includegraphics[width=100pt]{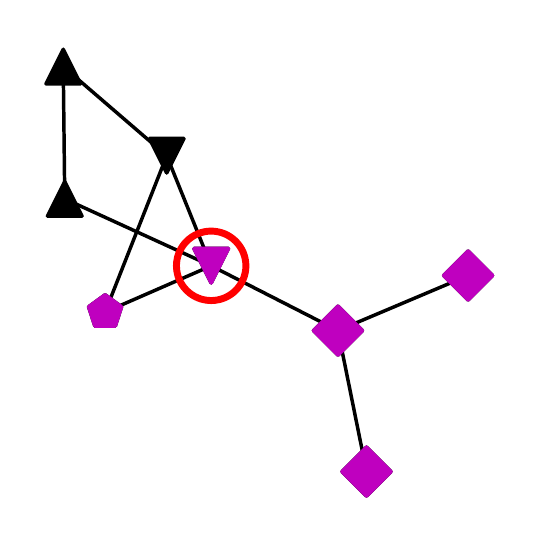}& \includegraphics[width=100pt]{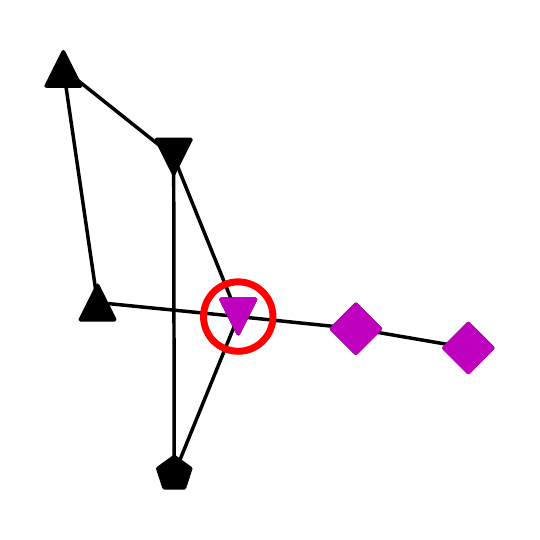}
     & \includegraphics[width=100pt]{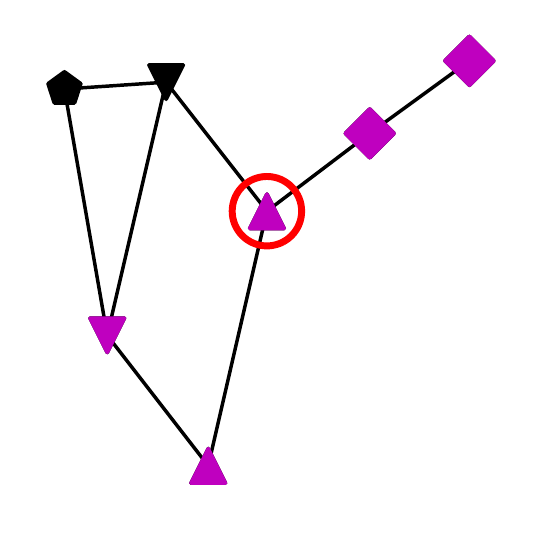}
     & \includegraphics[width=100pt]{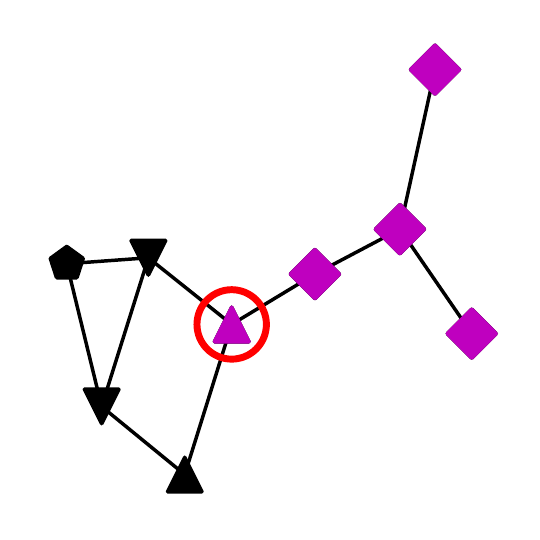} \\
         \gnnexp{} & SubgraphX &\gnnexp{} & SubgraphX \\
         \hline
     \includegraphics[width=100pt]{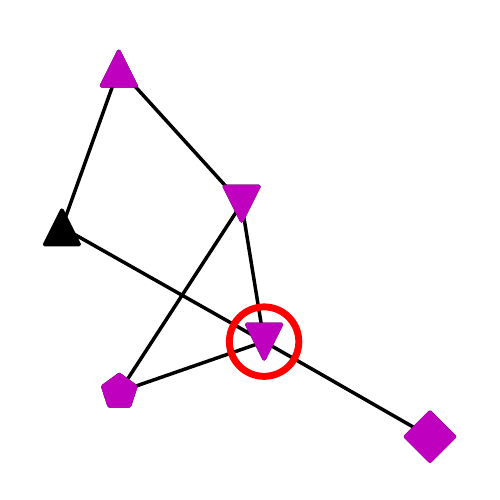}&
     \includegraphics[width=100pt]{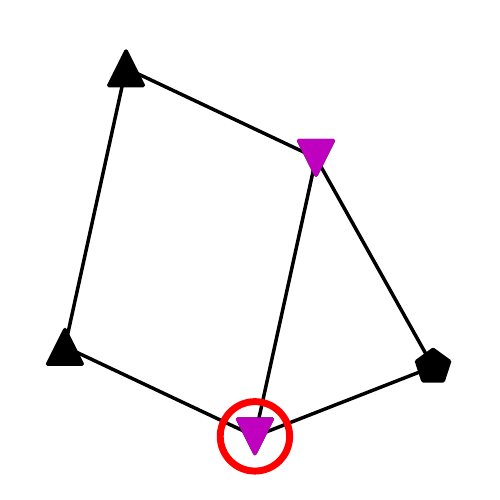}&
     \includegraphics[width=100pt]{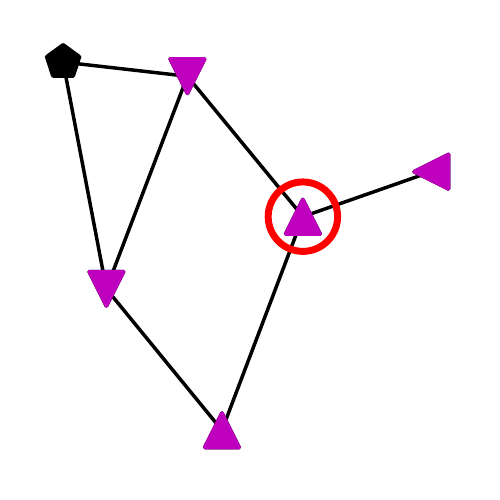} & \includegraphics[width=100pt]{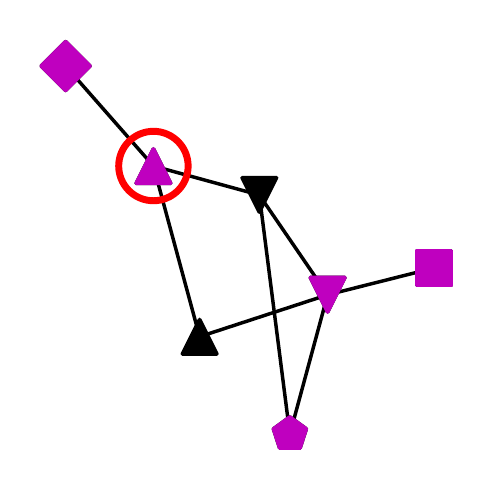}\\
         PGM & PGE & PGM & PGE \\
         \hline
         \includegraphics[width=100pt]{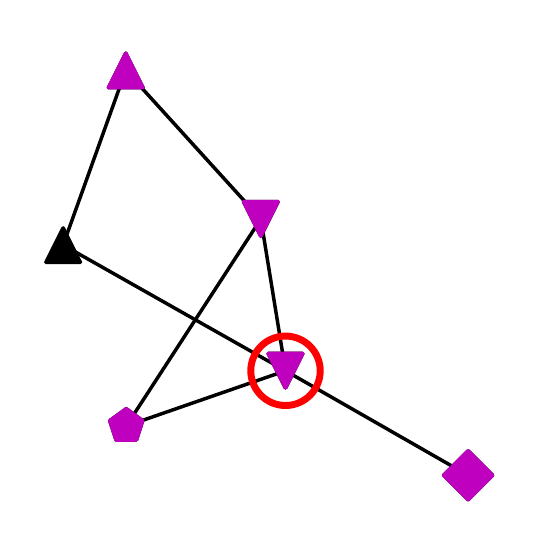}& \includegraphics[width=100pt]{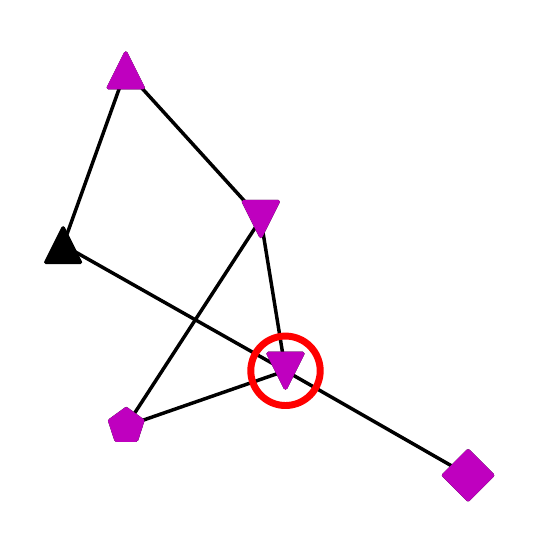}& \includegraphics[width=100pt]{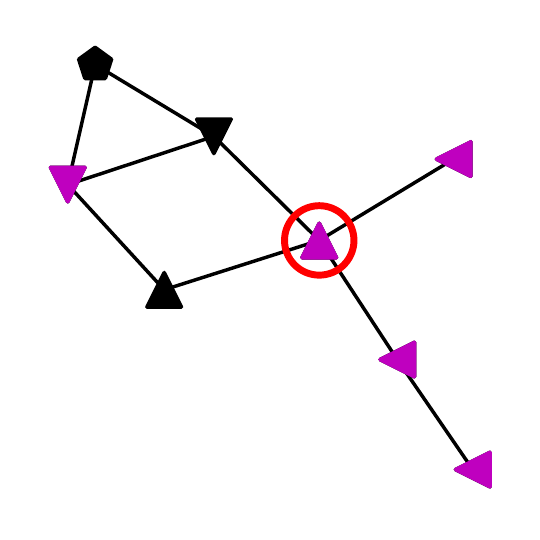} & \includegraphics[width=100pt]{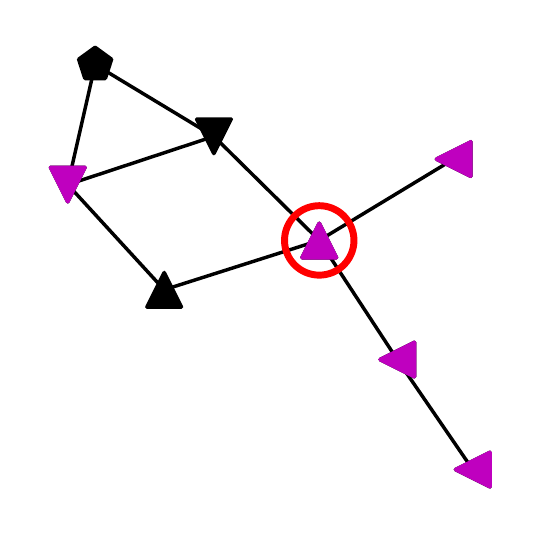}\\
         Grad & GradInput & Grad & GradInput \\
         \hline
         \includegraphics[width=100pt]{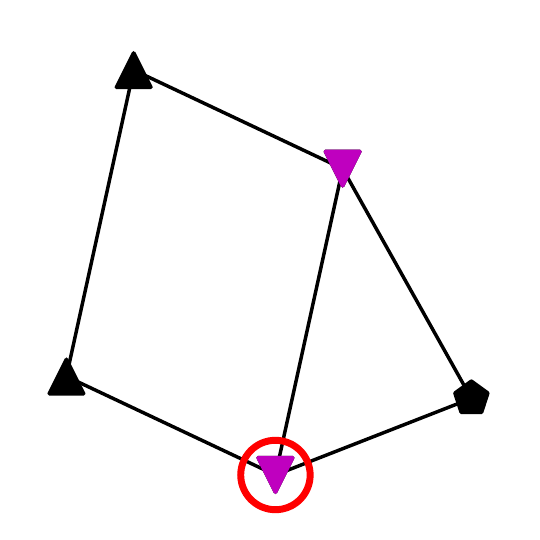}& \includegraphics[width=100pt]{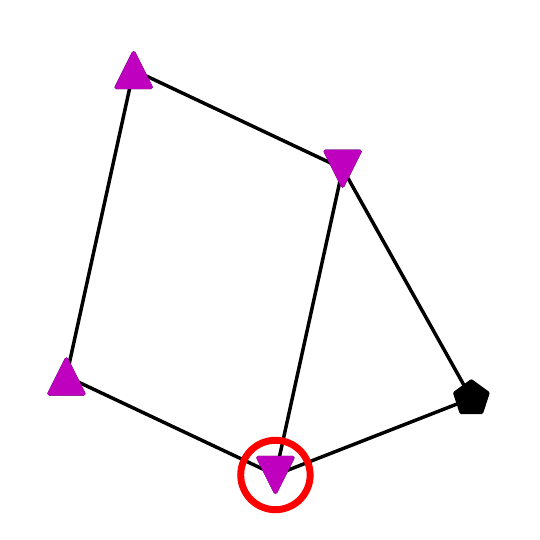}& \includegraphics[width=100pt]{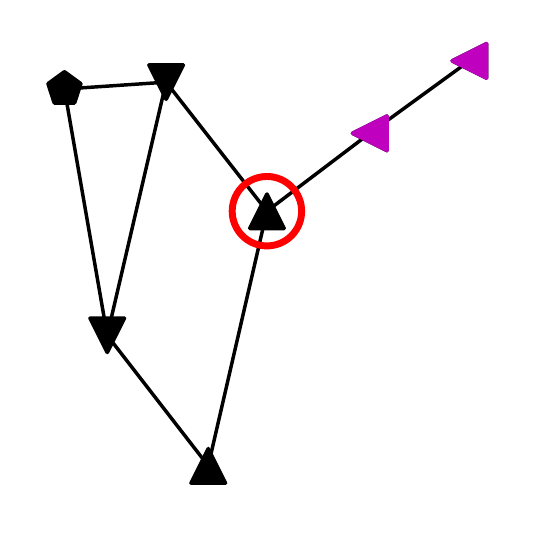} & \includegraphics[width=100pt]{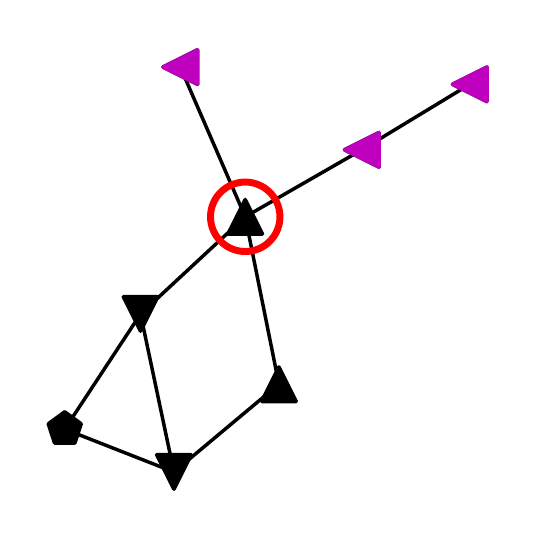}\\
         \approach{} $\tau=0.85$ & \approach{} $\tau=0.98$ & \approach{} $\tau=0.85$ & \approach{} $\tau=0.98$ \\
         \hline
    \end{tabular}
    \caption{Examples of explanations for two nodes of the synthetic dataset. 
    }
    \label{fig:syn2_examples}
\end{figure*}
 
\end{document}